\documentclass[journal]{IEEEtran}

\ifCLASSOPTIONcompsoc
  \usepackage[nocompress]{cite}
\else
  \usepackage{cite}
\fi

\ifCLASSINFOpdf
  \usepackage[pdftex]{graphicx}
  \graphicspath{{../pdf/}{../jpeg/}}
  \DeclareGraphicsExtensions{.pdf,.jpeg,.png}
\else
\fi

\usepackage{amsmath}

\usepackage{algorithmic}

\usepackage{array}

\usepackage{url}

\usepackage{amssymb}
\usepackage{color}
\usepackage{hyperref}

\newtheorem{theorem}{Theorem}
\newtheorem{corollary}{Corollary}

\newtheorem{lemma}{Lemma}

\newtheorem{proof}{Proof}
\newtheorem{assumption}{Assumption}

\newcommand{\conv}[2]{k_{#1}^{#2 *}}

\newcommand{\yzp}{\textcolor{black}}

\begin{document}
%
\title{Audio-Driven Talking Face Video Generation with Dynamic Convolution Kernels}
%
%

\author{Zipeng Ye,
        Mengfei Xia,
        Ran Yi,
        Juyong Zhang,~\IEEEmembership{~Member,~IEEE},
        Yu-Kun Lai,~\IEEEmembership{~Member,~IEEE},
        Xuwei Huang,
        Guoxin Zhang,
        Yong-Jin Liu,~\IEEEmembership{Senior Member,~IEEE}
\IEEEcompsocitemizethanks{\IEEEcompsocthanksitem 
Z. Ye, M. Xia, Y.-J. Liu are with MOE-Key Laboratory of Pervasive Computing, the Department of Computer Science and Technology, Tsinghua University, Beijing, China.
\IEEEcompsocthanksitem R. Yi is with the Department of Computer Science and Engineering, Shanghai Jiao Tong University.
\IEEEcompsocthanksitem J. Zhang is with the School of Mathematical Sciences,
University of Science and Technology of China.
\IEEEcompsocthanksitem Y.-K. Lai is with School of Computer Science and
Informatics, Cardiff University, UK.
\IEEEcompsocthanksitem X. Huang and G. Zhang is with Kwai Inc.}
\thanks{Y.-J. Liu and R. Yi are corresponding authors. E-mail: liuyongjin@tsinghua.edu.cn.}}

%

\markboth{To appear in IEEE Transactions on Multimedia, \MakeLowercase{https://doi.org}/10.1109/TMM.2022.3142387}%
{Shell \MakeLowercase{\textit{et al.}}: Bare Demo of IEEEtran.cls for IEEE Journals}
%
\maketitle

\begin{abstract}
In this paper, we present a dynamic convolution kernel (DCK) strategy for convolutional neural networks. Using a fully convolutional network with the proposed DCKs, high-quality talking-face video can be generated from multi-modal sources (i.e., unmatched audio and video) in real time, and our trained model is robust to different identities, head postures, and input audios. 
Our proposed DCKs are specially designed for audio-driven talking face video generation, leading to a simple yet effective end-to-end system. We also provide a theoretical analysis to interpret why DCKs work.
Experimental results show that our method can generate high-quality talking-face video with background at $60$ fps. Comparison and evaluation between our method and the state-of-the-art methods demonstrate the superiority of our method.
\end{abstract}


\IEEEpeerreviewmaketitle

\section{Introduction}\label{sec:introduction}

\IEEEPARstart{T}ALKING-FACE video refers to video which mainly focuses on head or upper body of the speaker given audio or text signals. It has wide range of applications in news, TV shows, commercials, online chat, online courses, etc. According to the types of input signals, there are text-driven (e.g.,~\cite{FriedTZFSGGJTA19}), audio-driven (e.g,~\cite{ChenMDX19, ChungJZ17, SongZLWQ19, ZhouLLLW19, WilesKZ18, FanWSX15, VougioukasPP19, yi2020audio, song2020everybody, thies2020neural}) and video-driven (e.g.,~\cite{KimCTXTNPRZT18, ThiesZSTN16, Averbuch-ElorCK17, PumarolaAMSM18, WilesKZ18, ZakharovSBL19, ZhangZHLLL19}) talking-face systems. In this paper, we propose an audio-driven talking-face system, capable of transferring the input talking-face video to a generated one corresponding to the input audio. It is a naturally cross-modal task with video and audio (i.e., of visual and auditory modalities) as input. The two modalities are strongly correlated~\cite{Nazzaro70}, and thus it is possible to drive the talking-face video using an audio.


\yzp{In this paper, we consider multi-modal fusion in a generation task, i.e., audio-driven talking-face video generation. For this task, a direct way is to treat multi-modal input as different features. To align these features, we can rearrange the audio features as additional channels of image frames and concatenate them with the image features. 
However, this method maps the audio feature elements to fixed locations, and as we will later show in our experiments presented in Section \ref{subsec:ablation}, it only works under special conditions where all the frames are aligned (i.e., each frame containing a frontal face at a fixed position), which are difficult to meet in practice.
Another possible way is to use landmark points or parametric models~\cite{BlanzV99, CaoWZTZ14} as a prior, which can be inferred from the audio sequence. Facial landmarks are highly correlated to expression but also sensitive to head pose, view angle and scale. Therefore, it is necessary to align input photo/frames with a standard face, which has challenges dealing with the following: (1) facial image fusion with background, (2) head motion and (3) extreme head pose. 3D parametric models can be used as a strict and precise prior, which preserves almost all the information of expression and lip motion, and we can render an image using the parametric model. However, parametric models only contain low frequency information and the rendered images are often not photo-realistic. Therefore, post-processing is needed, which makes the pipeline complex and time-consuming. On the other hand, using these priors, it is difficult to design an end-to-end system with a fully convolutional neural network (FCNN), which is desired to ensure generalizability.}

To overcome these drawbacks, in this paper, we propose a novel dynamic convolution kernel (DCK) technique that works well with FCNN for multi-modal generation tasks. Our key idea is to use a network to infer DCKs in a FCNN from audio modality. Then this FCNN can work with diverse input videos that have different head poses. Our model, i.e., FCNN with DCKs, is a network with dynamic parameters. In the literature, a few dynamic CNN parameter methods existed ~\cite{HaDL17, esquivel2019adaptive, li2018high, nie2019dynamic, chen2020dynamic}. However, they were all proposed for processing single modal information and due to limited adaptivity, they are difficult to be extended to handle cross-modal applications. See Section \ref{ssec:dynamic-para} for more details. Our DCKs are totally different in both purpose and content: (1) DCKs are for multi-modal tasks, where the kernel is inferred from input audio, and (2) DCKs use completely flexible kernels, and are linear once the kernels are determined.

In this paper, we consider the following characteristics in our audio-driven talking-face system: (1) \emph{Real time:} the video can be generated online when the audio signal is available; (2) \emph{High quality:} the quality of generated video should be good enough such that people cannot easily distinguish between real video and generated video; (3) \emph{Identity preserving:} the identity of the generated video should be preserved with the input video or photo; (4) \emph{Expression and voice synchronization:} the expression and lip motion of the generated video should be synchronized with the input audio; (5) \emph{Head motion:} the head pose and head motion of the generated video should be natural. 

To address these characteristics, we propose DCKs and use them to build an end-to-end and one-for-all system, which only needs to be trained once and can work for different identities. To make better use of the multi-modal inputs which are difficult to fuse, we design DCKs which are different from traditional static convolution kernels. Once the model is trained, traditional convolution kernels no longer change. In contrast, our DCK will change with different inputs. We \yzp{use the pre-trained audio network~\cite{Wav2Lip} to extract audio features and train a fully connected network} to infer the DCKs from the input audio, and therefore we can design a fully convolutional network for video with different audio inputs well handled. We adapt the U-net~\cite{RonnebergerFB15} for DCKs by replacing convolutional kernels at selected layers to DCKs. Furthermore, \yzp{we propose a novel dataset (including real videos and synthetic videos) to train our model in a supervised way.}

In summary, the main technical contributions of our work include: 
\begin{itemize}
\item We propose DCKs as an effective way to generate high-quality talking-face video from multi-modal input in {\it real time} with background and natural head motion, which is simple yet effective. 
\item We provide a theoretical analysis to explain DCKs' effectiveness. 
\item \yzp{We propose a novel mixed dataset, including both real videos and synthetic videos, to supervise the training of our model.}
\end{itemize}

\section{Related Work}

\subsection{Multi-modal Fusion}

One key challenge in tasks with multi-modal input is how to effectively fuse features in them.
In various engineering fields, many algorithms have been proposed for fusing features collected from different types of sensors, which may have different modalities, rates, formats or confidence levels. The Kalman filter~\cite{kalman1960new} is a classical algorithm for multi-sensor fusion. Bayesian inference~\cite{hall1997introduction} is another classic technique to fuse different features. For full details of existing fusion methods, the reader is referred to recent surveys \cite{atrey2010multimodal,BaltrusaitisAM19} and references therein.

Our study focuses on the neural network techniques. In this domain, a simple way for feature fusion is to directly concatenate features. The other simple way is to use different networks for extracting features of different modalities and use late feature fusion. The two simple strategies work well in classification and regression tasks, and \yzp{achieve} successes in many applications (e.g.,~\cite{hori2017attention, nojavanasghari2016deep, lin2021orthogonalization}). Video is the most common input with different modalities. For classification and regression tasks of video, some learning-based methods~\cite{wang2020knowledge, kuang2020deep, buitelaar2018mixedemotions, jiang2018modeling} are proposed, which design network structures for fusing multi-modal input. On the other hand, the talking face video generation is a generation task, which is quite different from classification and regression tasks, and the simple concatenation strategy often fails. The methods which are designed for classification and regression tasks also fail. In some image generation tasks, using landmarks and 3D models as priors is useful to fuse multi-modal input (e.g.,~\cite{ChenMDX19, yi2020audio, song2020everybody, thies2020neural, wen2020photorealistic}). However, as we mention in Section \ref{sec:introduction}, it is difficult to design a fully CNN and an end-to-end system using 2D or 3D prior. In this paper, we propose a novel fully convolutional network with DCKs for the talking face video generation task with multi-modal input.

\subsection{Neural Networks with Dynamic Parameters}
\label{ssec:dynamic-para}

The new model proposed in this paper, i.e., the fully convolutional network with DCKs, is a neural network with dynamic parameters. In recent years, several research works on designing neural networks with dynamic parameters have been proposed. The HyperNetworks~\cite{HaDL17} uses a hypernetwork to generate the weights for the other network, which has the similar idea as ours, but their motivations (for language modeling) and network structures (using recurrent neural network) are completely different from ours.

For CNNs, although fixed kernels are dominant in most research, there exist adaptions of CNNs (\cite{esquivel2019adaptive, li2018high, nie2019dynamic, chen2020dynamic}) whose kernels can be dynamically adjusted. However, all these methods can only handle single mode information as input. The work \cite{esquivel2019adaptive} is designed for a classification task, which estimates a set of weights from input and these weights are used for balancing the output of nine sub-network-structures. Although the weights can be dynamically set, only adjusting the weights of nine sub-structures has limited capacity (that is suitable for single mode input); while our DCKs can adaptively set up to $10^4$ parameters, which are more powerful and suitable for multi-modal input. The work \cite{chen2020dynamic} is similar to \cite{esquivel2019adaptive} in the spirit of using dynamic weights to adjust the linear combinations of a few convolution layers. The works \cite{li2018high, nie2019dynamic} predict a feature from input (different input may lead to different feature) and use this feature to do convolution at fixed positions in the network. As a comparison, our DCKs can predict parameters in different positions in the network and thus \yzp{are} more \yzp{flexible}.

\begin{figure}[ht]
\centering
\includegraphics[width=\columnwidth]{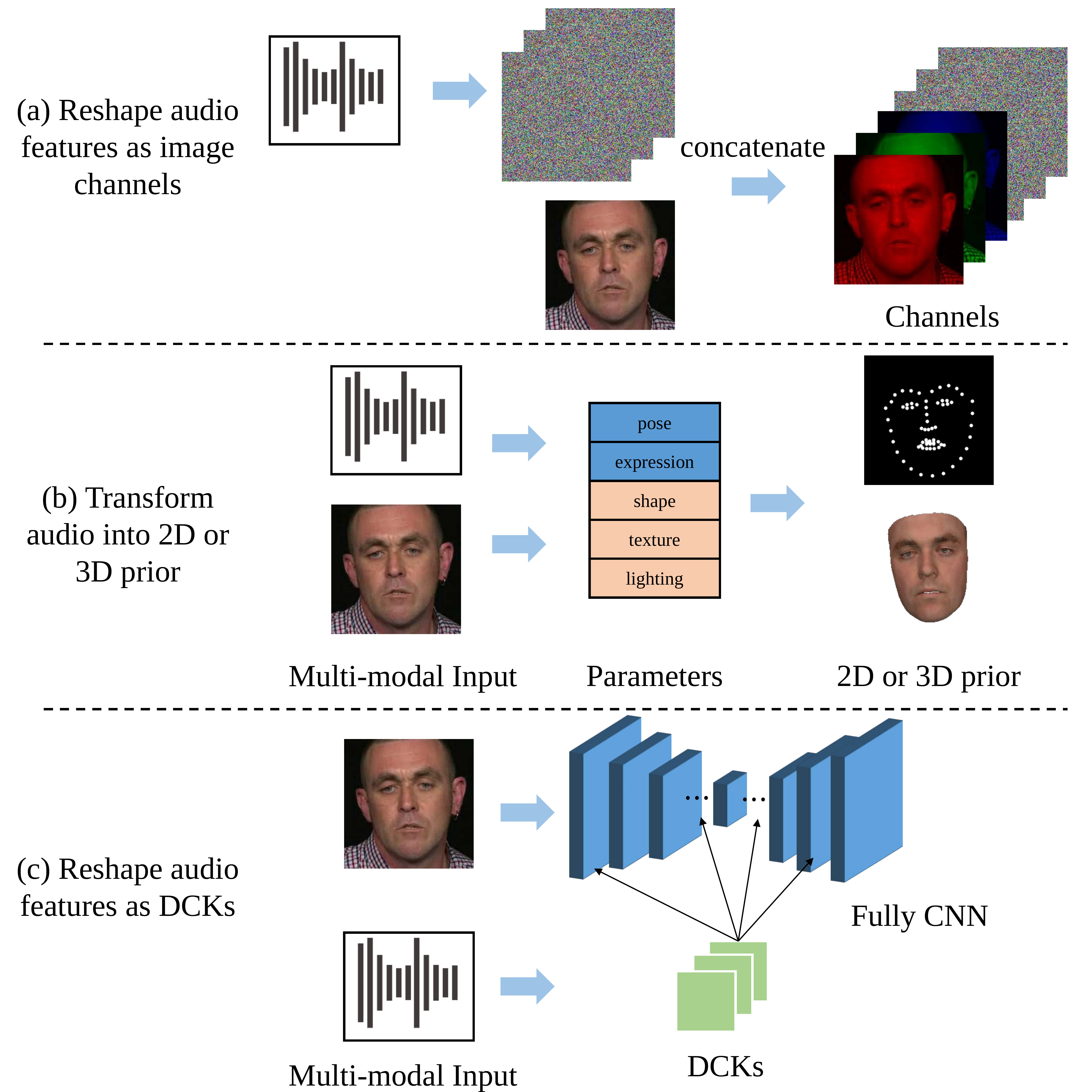}
\caption{Three strategies for fusing multi-modal input. (a) A simple and direct strategy is to extract the features of different modes, and concatenate them to feed into a network~\cite{ChungJZ17, ZhouLLLW19}. (b) The second strategy is to use 2D~\cite{ChenMDX19} or 3D~\cite{yi2020audio, wen2020photorealistic, thies2020neural, song2020everybody} prior. It infers the prior's parameters from the audio input. (c) We propose a novel strategy, which extracts features from audio input and reshape features as DCKs of fully convolutional network.}
\label{fig:feature_fusion}
\end{figure}

\subsection{Audio-driven Talking-face Video Generation}

Audio-driven talking-face video generation is a task that uses an audio to drive a specified face (from a face photo or a talking face video) and produce a talking face video, with focus on fine talking details in head or upper body of the speaker \cite{ChenMDX19, ChungJZ17, SongZLWQ19, ZhouLLLW19, WilesKZ18, FanWSX15, VougioukasPP19, yi2020audio, song2020everybody, thies2020neural, wen2020photorealistic}. It is a typical task that uses multi-modal input. Many previous works use facial landmarks or 3D morphable models (3DMM)~\cite{BlanzV99} as priors to bridge the two modalities. Inferring the prior from audio and using it to generate talking face video have been used in some practical methods \cite{ChenMDX19, yi2020audio, wen2020photorealistic, thies2020neural, song2020everybody}. 
However, it is difficult to use 2D or 3D prior to design an end-to-end system which can address the five characteristics summarized in Section \ref{sec:introduction}.

In \cite{SongZLWQ19,ZhouLLLW19} two end-to-end methods are proposed that do not use 2D or 3D priors. The work \cite{ZhouLLLW19} generates talking face video by the composition of a subject-related part and a speech-related part. Based on this composition, they propose an encoder-decoder model to disentangle the two parts from the video and use the input audio as the speech-related part to generate video. The work~\cite{SongZLWQ19} proposes a conditional recurrent generation network which uses an audio and an image as input, and output a video. Both methods can only work with a fixed standard head pose and output a video without head motion. As a comparison, our method can work with different poses and generate natural head motion.

\section{Dynamic Convolution Kernels (DCKs)}

\subsection{Motivation}

In this paper, we deal with a talking-face video generation task whose input is a pair of unmatched audio and video. The input contains entirely different modalities which have different forms and contents. How to fuse them together to effectively guide the training process is not easy and many works have studied this. As discussed in Section \ref{sec:introduction}, our target is to generate high-quality video with head motion. How to design the fusion is key to achieving these targets. Currently, there are two popular strategies to perform the fusion:

\begin{figure}[ht]
\centering
\includegraphics[width=\columnwidth]{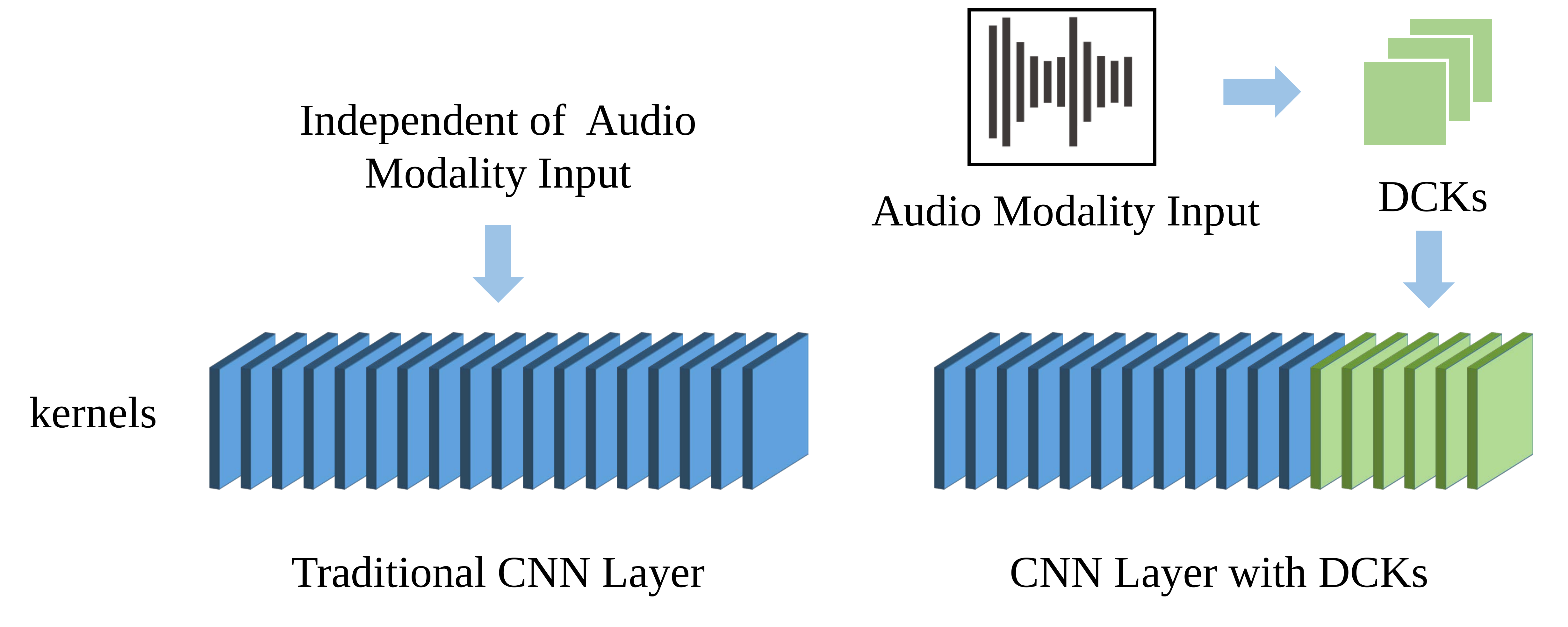}
\caption{Illustration of a CNN layer with DCKs. The blue are static convolution kernels and the green are dynamic convolution kernels. We reshape the features of audio modality into the shape of convolution kernels and use them to complete the CNN layer.}
\label{fig:DCKs_illustration}
\end{figure}

The first strategy for multi-modal input is to extract their features, and concatenate them or input them together to a network \cite{ChungJZ17, ZhouLLLW19} (Figure \ref{fig:feature_fusion}a). For example, we can use encoders to transform them into vectors and use a decoder to generate a video, which is not a fully convolutional network. We can also reshape the features of the input audio as channels of an image and concatenate them with the image, which works in a special case that the images are aligned so each pixel position has a fixed content. However, it is hard to achieve success in general because images are usually not aligned strictly and there is no fixed semanteme at each pixel. Our experimental results in Section \ref{subsec:ablation} demonstrate this observation.

The second strategy is to use a parametric model as 3D prior (e.g.,~\cite{yi2020audio, wen2020photorealistic, thies2020neural, song2020everybody}) or use facial landmarks as 2D prior (e.g.,~\cite{ChenMDX19}). As shown in Figure \ref{fig:feature_fusion}b, they use audio to predict the prior and then use the prior to conduct the generation of videos. Using facial landmarks requires alignment of images and using 3D prior usually leads to high time cost, as demonstrated by our experimental results in Section \ref{subsec:stoa}.

To directly output high-quality video frames, a fully convolutional network is usually preferred to ensure generalizability. In this work, we design dynamic convolution kernels (DCKs), which are different from traditional static convolution kernels (Figure \ref{fig:feature_fusion}c). Once the model is trained, traditional convolution kernels no longer change. In contrast, our DCK is designed to infer from different inputs and therefore can change during the inference process (Figure~\ref{fig:DCKs_illustration}). We use the convolution kernel as part of the generative network, which is dynamic for different input audios.

\subsection{The Structure of DCKs}

A traditional fully convolutional network $f(x)$, whose convolution kernels are $K(f) = \{ k_1, k_2, \dots, k_n \}$, is a transformation that iteratively applies the convolution operation to the input $x$. Denote by $y_0, y_1, \dots, y_n$ the intermediate results where $x=y_0$ is the input and $y_n$ is the output, we have:
\begin{equation}
y_{i} = g_i(k_i^{\star} y_{i-1})\quad for \quad i=1, 2, \dots, n,
\end{equation}
where $k_i^{\star}$ is the the $i$th convolution operator whose kernel is $k_i$ and $g_i$ is the $i$th combination of normalization and activation function. In this case, all convolution kernels are static which are learned from the training set and will not change in inference. 

\begin{figure*}[ht]
\centering
\includegraphics[width=\textwidth]{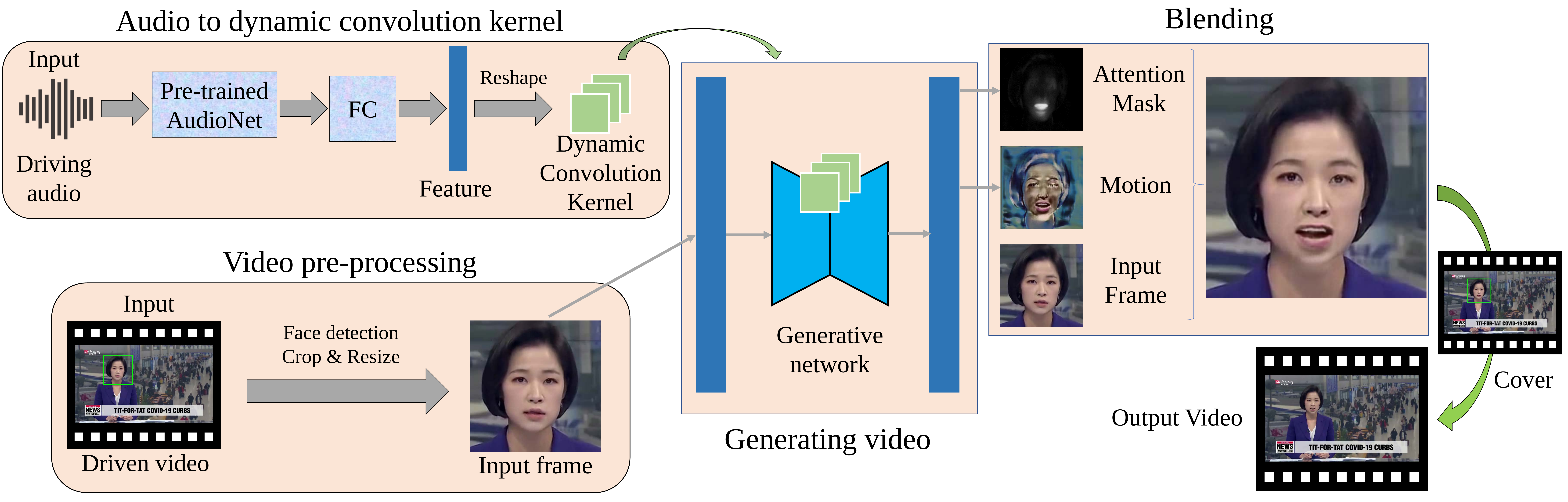}
\caption{The pipeline of our system and the architecture of our model. We adapt the U-net by incorporating dynamic convolution kernels as our generative network. We \yzp{use the pre-trained audio network~\cite{Wav2Lip} to extract audio features and train a fully connected network} to infer the dynamic convolution kernels from the input audio and use them to replace some traditional static convolution kernels. We detect the facial area from input video and crop it as the input of the generative network. The outputs of the generative network are an attention mask and a motion image. We blend them with the input to obtain the generation result.}
\label{fig:pipeline}
\end{figure*}

We propose a fully convolutional network using DCKs. Some selected convolutional layers $K_d \subset K(f)$ are no longer fixed after training, but instead are determined based on the input audio via a neural network, i.e. $k_j = h_j(A)$, for $k_j \in K_d$, where $h_j$ is a neural network to determine the $j$th DCK given the input audio $A$. More generally, denote by $\Theta(f)$ all the parameters of a fully convolutional network, i.e. parameters of all convolution kernels, and some selected parameters $\Theta_d(f) \subset \Theta(f)$ are no longer fixed after training. A traditional fully convolutional network $f(x)$ can be written as $f(x; \Theta(f))$. The corresponding network with DCKs is $f(x; \Theta_s(f), \Theta_d(f))$, where $\Theta_s(f) = \Theta(f)\setminus \Theta_d(f)$ is the rest static convolution kernels. \yzp{We use a network to dynamically generate $\Theta_d(f)$ based on features of the input audio. We (1) use the pre-trained audio network in Wav2Lip~\cite{Wav2Lip}, which consists of 2D convolutional layers and residual blocks, to extract audio features from Mel Spectrogram of input audio and (2) train a fully connected network to infer $\Theta_d(f)$ from the audio features. We have:}
\begin{equation}
\Theta_d(f) = h_{2}(h_{1}(A)),
\end{equation}
\yzp{where $A$ is the Mel Spectrogram of input audio, $h_{1}$ is the pre-trained audio network~\cite{Wav2Lip} and $h_{2}$ is the fully connected network.} We reshape the output of the audio network into the shape of convolution kernels and use them to complete the fully convolutional network. Therefore, the fully convolutional network is dynamic with different input audios.

The advantages of using DCKs include the following aspects: (1) We can design a fully convolutional network for the input video, leading to real-time performance; (2) The convolution kernel is dynamic and can effectively fuse features from multi-modal inputs; (3) There is no binding between features and positions of pixels so it can work in different poses and different translations. 

We present a theoretical interpretation for DCKs in Section \ref{sec:DCK-interpretation}, to explain why it is useful for the cross-modal talking face video generation task.

\section{The System}
\label{sec:system}

We propose an audio-driven talking face video generation system, whose inputs are unmatched audio and video, and the output is a synthetic video. The pipeline of our system is shown in Figure \ref{fig:pipeline}. It is an end-to-end approach by directly outputting the synthetic video without intermediate results. Our system can generate high quality results in real time by efficiently incorporating the DCK technique and a \yzp{supervised} training scheme.

\subsection{Fully Convolutional Network with DCKs}

Our system deals with a multi-modal generation task whose input includes both audio and video. We \yzp{use the pre-trained audio network in Wav2Lip~\cite{Wav2Lip} to extract audio features from Mel Spectrogram of input audio, and train (1) a fully connected network to generate DCKs from audio features, and (2) a fully convolutional network with DCKs.
For training the two networks, we propose a novel method to train our model in a supervised way.}

\subsection{Training}
\label{subsec:trainning}

In our task, the inputs are a pair of unmatched video and audio, and the output is a synthetic video. Denote by $\mathcal{V}$ the space of talking-face video and $\mathcal{A}$ the space of audio of talking-face video. An audio-driven talking-face system is a function $f: \mathcal{V} \times \mathcal{A} \rightarrow \mathcal{V}$. For any $A \in \mathcal{A}$ and $V \in \mathcal{V}$, $f(V, A)$ is a synthetic video, which have the same identity as $V$ and the same expression (including lip motion) as $A$. 

We use a supervised training scheme to train our model. Ideally, we need a training set consisting of pairs of talking-face videos which have different lip motions and the same other attributes (including identity and head motion) to train our model. However, it is difficult to obtain this kind of training dataset of real videos because the condition is too strict: even in a real talking face video without head movement, it is hard to extract two frames with exactly the same head pose. We take an alternative approach that synthesizes videos and pairs them with real videos to build this kind of dataset. Some talking face generation methods~\cite{Wav2Lip} can generate a talking-face video from a reference video and an audio, where the generated video has the same identity and head motion as the reference video. We use the method~\cite{Wav2Lip} to generate a new training dataset by the following steps: (1) we collect $N_r=550$ real talking-face videos $\{V_1^0, V_2^0, \dots, V_{N_r}^0\}$ from video websites and collect $m=550$ audios from talking-face videos, whose lengths are about 60 seconds; (2) for each real video $V_i^0, i=1, 2, \dots, N_r$, we use the method~\cite{Wav2Lip} to generate $m$ videos $V_i^{1}, V_i^{2}, \dots, V_i^{m}$ which has the same identity and head motion with the video $V_i^0$ and different lip motions; (3) we combine $V_i^0$ with $V_i^j$ as a pair of videos and then we have $mN_r$ pairs of videos.

We use the dataset obtained above (including real videos and synthetic videos) to train our model. In the training set, all talking-face videos have their corresponding audios and we denote their relation by an operator $A(V): \mathcal{V} \rightarrow \mathcal{A}$ which maps a talking-face video $V$ to its corresponding audio $A$. For each batch, we randomly select a pair of videos $V_1, V_2$ and their corresponding audios $A_1=A(V_1), A_2=A(V_2)$. We use our model to generate video $f(V_1, A_2)$ from $V_1$ and $A_2$ and $f(V_2, A_1)$ from $V_2$ and $A_1$.

\textbf{Reconstruction Loss.} We consider $f(V_1, A_2), f(V_2, A_1)$ are generation results and $V_2, V_1$ are their ground truth. Ideally, the generation results should be exactly the same as the ground truth. We use reconstruction loss to constrain our model to generate talking face videos similar to the ground truth. The loss term is calculated as the $L_1$ norm of the difference between generation results and the ground truth, i.e.,
\begin{equation}
\begin{array}{l}
L_{rec}(f) \\= \mathbb{E}_{V_1, V_2 \sim \mathcal{V}}(\lVert f(V_1, A_2) - V_2 \rVert_1 + \lVert f(V_2, A_1) - V_1 \rVert_1).
\end{array}
\label{eq:rec-loss}
\end{equation}

\textbf{Adversarial Loss.} We use adversarial loss to ensure that $f(\mathcal{V}, \mathcal{A})$ has the same distribution as $\mathcal{V}$, which can improve the quality of generation results. We adapt the adversarial loss of LSGAN~\cite{mao2017least} as:
\begin{equation}
\begin{array}{l}
L_{adv}(f, D) \\= \mathbb{E}_{V_1, V_2 \sim \mathcal{V}}(\lVert D(f(V_1, A_2))\rVert_2^2) + \mathbb{E}_{V \sim \mathcal{V}}(\lVert 1 - D(V)\rVert_2^2).
\end{array}
\end{equation}

The overall loss function is in the following form:
\begin{equation}
\begin{array}{l}
L_{total}(f, D) = L_{adv}(f, D) + \lambda_{rec}L_{rec}(f),
\end{array}
\end{equation}
where $\lambda_{rec}$ is the weight for balancing the multiple objectives. For all experiments, we set $\lambda_{rec} = 10$. The optimization target is:
\begin{equation}
\min_{f} \max_{D} L_{total}(f, D).
\end{equation}

\subsection{Blending}

\begin{figure*}[htb]
  \centering
  \includegraphics[width=\textwidth]{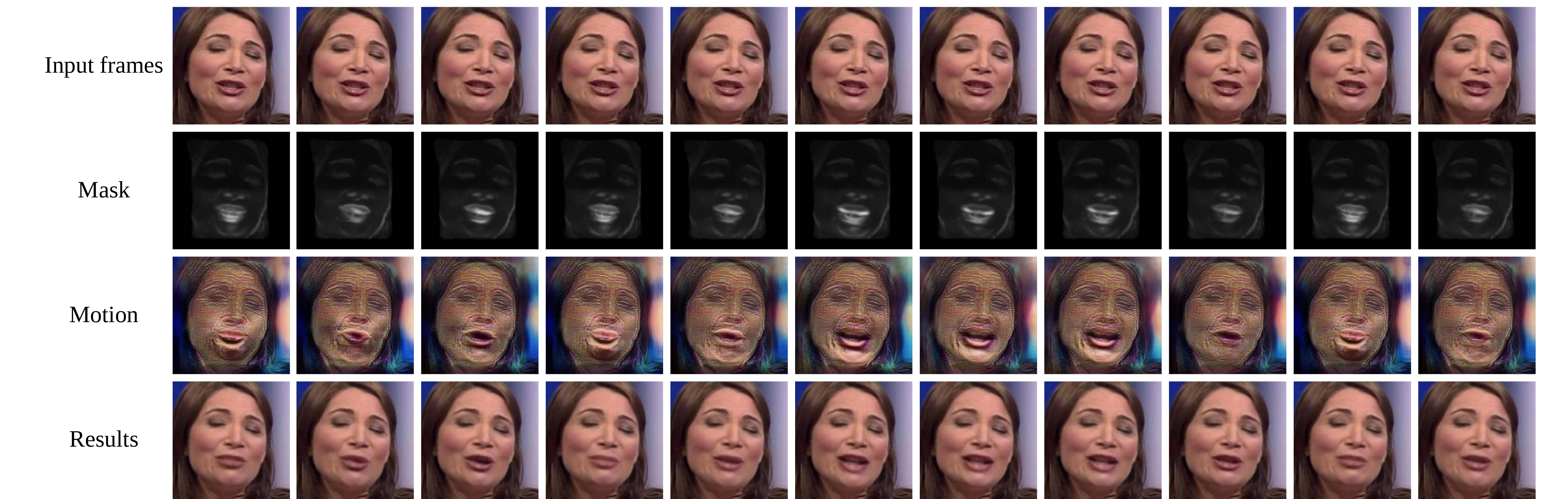}
  \caption{
\yzp{Results of attention masks and motion images generated by our method. The audio input is 'for many of us', and the video input is obtained by repeating a face photo multiple times.}  }
 \label{fig:mask_color_show}
\end{figure*}

Instead of directly generating the synthetic frames, the output of our method (shown in Fig.~\ref{fig:mask_color_show}) is an attention mask $\alpha$ which is a grayscale image, and a motion image $M$ which is a color image that presents the change. $\alpha$ determines at each pixel how much the output should be influenced by the motion image $M$.
Denote by $I$ the input image and $I'$ the synthetic image, we have:
\begin{equation}
I' = I \otimes (1 - \alpha) + M \otimes \alpha,
\end{equation}
where $\otimes$ is pixel-wise multiplication.

Compared with directly generating the synthetic frames, our method has the following advantages. It can not only enforce the network to focus on audiovisual-correlated regions but also offers an efficient post-processing to produce desired output, avoiding expensive image fusion. It also increases the interpretability of the network and we can be informed of where the network focuses. In practice, it is difficult to train the network with DCKs by directly generating the synthetic frames. In Figure~\ref{fig:ablation_study}, the generation results of directly generating the synthetic frames have different skin color from inputs whereas those of blending have the same skin color as inputs.

\begin{figure}[ht]
\small
\centering
\includegraphics[width=0.49\textwidth]{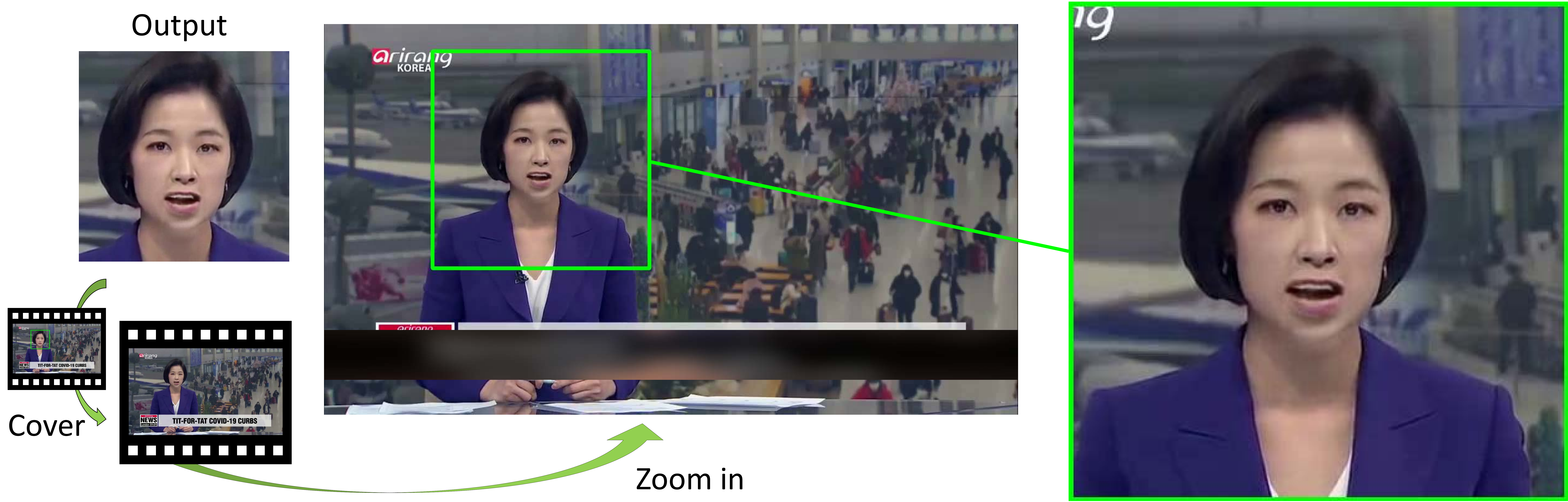}
\caption{Our system directly covers the generated results onto the background instead of image fusion. No visible artifacts are observed on the boundary between the face and background.}
\label{fig:background}
\end{figure}

{
\subsection{Adding Background in Real Time}

Our system does not need to use image fusion to integrate the generated results with the background, which usually takes a long time. We directly cover the generated results onto the background instead of image fusion and this helps save time. The boundaries of generation results are the same as those of inputs, because we do not change the head pose and non-face area. Due to the attention mechanism and the mask loss, the direct covering has good performance and there are no visible artifacts on the boundary between the face and background. A frame with background generated by direct covering is shown in Figure~\ref{fig:background}. 

As a comparison, other state-of-the-art methods (e.g., \cite{ChenMDX19, ChungJZ17, Wav2Lip, WilesKZ18, ZhouLLLW19, VougioukasPP18, yi2020audio, song2020everybody, thies2020neural, wen2020photorealistic}) either cannot keep the head pose or change the non-face area, and therefore, the face and background have different color values on the boundary. 
We note that in these methods, directly covering would cause inconsistency between the generated region and background. Therefore, they 
cannot adopt our directly covering scheme as efficient as ours; e.g., Wav2Lip~\cite{Wav2Lip} also generates the facial region and uses the direct covering to add background, but as we will show in Section \ref{subsec:stoa}, their results exhibit clear boundaries between the generated region and background.
}

\section{Theoretical Interpretation for DCKs}
\label{sec:DCK-interpretation}

We can understand the dynamic convolution kernels (DCKs) in the following way. Denote by $\mathcal{T}$ the space of tasks such as all expressions. For a fixed task $t \in \mathcal{T}$ such as smiling, we can train a network to transfer an image to a new image with smile. For different tasks in $\mathcal{T}$, we can train different networks. We believe these networks are mostly the same with slight distinction. We use static convolution kernels to learn the common characteristics and use dynamic convolution kernels to learn the distinction. Therefore, the network with dynamic kernels can handle all tasks in the space $\mathcal{T}$.

We present the following formulations to provide a theoretical interpretation of the above statement. Some experimental validations are presented in Section \ref{subsec:DCK-validation}.

\subsection{Interpretation from set approximation}

\begin{assumption}
\label{assumption:one-for-all}
A multi-modal task can be fulfilled by solving a set $\mathcal{T}=\{t_i\}$ of simpler tasks, each of which $t_i$ can be fulfilled by a fully convolutional network $f^i$ with fixed parameters. All the networks $\{f^i\}$ have the same structure and most of their parameters are the same. 
\end{assumption}

Now we show that a single fully convolutional network with DCKs can well approximate the set of networks $\{f^i\}$ with bounded output errors.

\begin{lemma}
Let the activation function $g$ be any one of Leaky ReLU, $\tanh$ or Sigmoid. For any input $x^1,x^2\in\mathbb R^n$,
the following inequality holds:
\begin{equation}
\|g(x^1)-g(x^2)\|_p\leqslant\|x^1-x^2\|_p.
\end{equation}
\end{lemma}

\begin{proof}
For any $a\in\mathbb R$, $a-g(a)$ increases monotonically;
hence for $a^1\geqslant a^2, a^1,a^2\in\mathbb R$, we have $a^1-g(a^1)\geqslant a^2-g(a^2)$, i.e., $|a^1-a^2|=a^1-a^2 \geqslant g(a^1)-g(a^2)=|g(a^1)-g(a^2)|$. 
That completes the proof.
\end{proof}

Given two fully convolutional networks $f^1$ and $f^2$ corresponding to two sets of convolution kernels $\{k_1^1,k_2^1,\cdots,k_n^1\}$ and $\{k_1^2,k_2^2,\cdots,k_n^2\}$, respectively, where each pair of $(k_i^1,k_i^2)$ has the same kernel size, let $\{y_0^1,\cdots,y_n^1\}$ and $\{y_0^2,\cdots,y_n^2\}$ be the two sets of
intermediate results of two convolution networks $f^1$ and $f^2$, where $y_0^1=y_0^2=x$. Then we have:
\begin{equation}
y_i^j = g_i(\conv{i}{j}y_{i-1}^j),\quad i=1,2,\cdots,n,\quad j=1,2,
\end{equation}
where $*$ is the convolution operator and $g_i$ is the activation function.
The following theorem gives an upper bound of the difference between the two outputs $f^1(x),f^2(x)$ in terms of the $L_p$ norm.

\begin{theorem}
\label{theorem:bound}
If all convolution
kernels have a uniform upper bound of their $L_p$ norm, i.e., $\|k_i^j\|_p\leqslant M_p$ for $\forall i,j$ and some $M_p>0$, the following inequality holds:
\begin{equation}
\|y_n^1-y_n^2\|_p\leqslant M_p^{n-1}\|x\|_p\sum_{i=1}^n\|k_i^1-k_i^2\|_p.
\label{eq:the1}
\end{equation}
\end{theorem}
\begin{proof}
We prove this theorem by induction. First we consider the case $n=1$, i.e., there is only 
one convolution kernel for each set. By calculating the $L_p$ loss, we have
\begin{equation*}
\begin{array}{l}
\|y_1^1-y_1^2\|_p = \|g(k_1^{1*}x)-g(k_1^{2*}x)\|_p \\
\leqslant \|\conv{1}{1}x-\conv{1}{2}x\|_p = \|(k_1^1-k_1^2)^*x\|_p \\
\leqslant \|k_1^1-k_1^2\|_p\cdot\|x\|_p,
\end{array}
\end{equation*}
where the last inequality comes directly from the Cauchy inequality.
Now, suppose the inequality (\ref{eq:the1}) holds for $n\leqslant m-1$. For $n=m$, we have
\begin{align*}
&\|y_m^1-y_m^2\|_p \\
= &\|g_m(\conv{m}{1}y_{m-1}^1)-g_m(\conv{m}{2}y_{m-1}^2)\|_p \\
\leqslant &\|(\conv{m}{1}y_{m-1}^1-\conv{m}{2}y_{m-1}^2)\|_p \\
\leqslant &\|\conv{m}{1}y_{m-1}^1-\conv{m}{1}y_{m-1}^2\|_p + \|\conv{m}{1}y_{m-1}^2-\conv{m}{2}y_{m-1}^2\|_p \\
\leqslant &\|k_m^1\|_p\cdot\|y_{m-1}^1-y_{m-1}^2\|_p + \|k_m^1-k_m^2\|_p\cdot\|y_{m-1}^2\|_p \\
\leqslant &M_p\cdot M_p^{m-2}\|x\|_p\sum_{i=1}^{m-1}\|k_i^1-k_i^2\||_p \\
&+ \|k_m^1-k_m^2\|_p\|x\|_p\prod_{i=1}^{m-1}\|k_i^2\|_p \\
\leqslant &M_p^{m-1}\|x\|_p\sum_{i=1}^m\|k_i^1-k_i^2\||_p
\end{align*}
That completes the proof.
\end{proof}

In practice, the constant $M_p$ is usually small, e.g., in all experiments in Section \ref{sec:experiment}, $M_p=0.4559$. Then 
Theorem \ref{theorem:bound} says that for two networks $f^1$ and $f^2$ with a fixed number $n$ of convolution layers,
\begin{equation}
    \|f^1(x)-f^2(x)\|_p\leqslant C_p\|x\|_p\sum_{i=1}^n\|k_i^1-k_i^2\|_p,
\end{equation}
where $C_p$ is a constant independent of the input $x$ and the convolution networks.

Note that although all the networks in the set $\{f^i\}$ have the same structure and most of their parameters are the same, the remaining parameters can be significantly different. Then any fully convolutional network with fixed parameters cannot well approximate all the networks in $\{f^i\}$. Let $f\in\{f^i\}$ and $f'$ be a fully convolutional network with DCKs which are inferred from the audio modality. If the inference makes the parameters in DCKs well approximate the parameters in the corresponding layers of $f$, $f'$ can well approximate any $f$ in $\{f^i\}$.

\subsection{Interpretation from loss values}

The value of the objective loss function can reflect the quality of generation results of the system to some extent. Next we present the error bounds of \yzp{two} loss terms in our objective function, showing the loss value is approximately optimal\footnote{Assume that the optimal loss value is provided by a network $f\in\{f^i\}$.} using a network with DCKs. \yzp{In addition, we also present the error bounds of cycle loss~\cite{zhu2017unpaired}, which is a useful loss term. }

\begin{corollary}
Let $f^j_D$, $j=1,2$, be the discriminator network, which is a fully convolutional network. Let $\{k_{n+1}^j,\cdots,k_{n+r}^j\}$ be the convolution kernels of the discriminator
$D^j$ (Here, the indexes follow those of the generator). The adversarial loss can be expressed as:
\begin{align*}
&L_{adv}(f^j,D^j) \\
=& \mathbb E_{x\sim\mathcal X}(\|D^j(f^j(x))\|_2^2)+\mathbb E_{x\sim\mathcal X}(\|1-D^j(x)\|_2^2) \\
=& \mathbb E_{x\sim\mathcal X}(\|f_D^j\circ f^j(x)\|_2^2)+\mathbb E_{x\sim\mathcal X}(\|1-f_D^j(x)\|_2^2).
\end{align*}
Then there exist constants $A_1,A_2>0$, such that the following inequality holds:
\begin{align*}
&|L_{adv}(f^1,D^1)-L_{adv}(f^2,D^2)|\\
\leqslant &\left(A_1\left(DK_2\right)^2+A_2 \cdot DK_2\right)\mathbb E_{x\sim\mathcal X}(\|x\|_2^2),
\end{align*}
where $DK_2 = \sum_{i=1}^{n+r}\|k_i^1-k_i^2\|_2$.
\end{corollary}
\begin{proof}
We can regard the discriminator $D^j$ as another convolution network together with convolution
kernels $\{k_{n+1}^j,\cdots,k_{n+r}^j\}$ and activation functions $g_{n+1},\cdots,g_{n+m}$. Notice
that there exist $C_1,C_2>0$, such that the inequalities hold:
$$\|f_D^2\circ f^2(x)\|_2\leqslant C_1\|x\|_2,\ \mbox{and}$$
$$\|1-f_D^2(x)\|_2\leqslant C_2\|x\|_2.$$
We denote
\begin{align*}
&\|f_D^1\circ f^1(x)-f_D^2\circ f^2(x)\|_2\leqslant\Delta_1\|x\|_2,\\
&\quad\|f_D^1(x)-f_D^2(x)\|_2\leqslant\Delta_2\|x\|_2,
\end{align*}
where
\begin{align*}
\Delta_1 &\propto \sum_{i=1}^{n+r}\|k_i^1-k_i^2\|_2,\\
\Delta_2 &\propto \sum_{i=1}^{r} \|k_{n+i}^1-k_{n+i}^2\|_2.
\end{align*}
Then we have
\begin{align*}
&|L_{adv}(f^1,D^1)-L_{adv}(f^2,D^2)| \\
\leqslant& \mathbb E_{x\sim\mathcal X}(|\|f_D^1\circ f^1(x)\|_2^2-\|f_D^2\circ f^2(x)\|_2^2|)\\
&+\mathbb E_{x\sim\mathcal X}(|\|1-f_D^1(x)\|_2^2-\|1-f_D^2(x)\|_2^2|) \\
\leqslant& \mathbb E_{x\sim\mathcal X}\left((\|f_D^2\circ f^2(x)\|_2+\Delta_1\|x\|_2)^2-\|f_D^2\circ f^2(x)\|_2^2\right) \\
&+ \mathbb E_{x\sim\mathcal X}\left((\|1-f_D^2(x)\|_2+\Delta_2\|x\|_2)^2-\|1-f_D^2(x)\|_2^2\right) \\
\leqslant& \mathbb E_{x\sim\mathcal X}(((C_1+\Delta_1)^2-C_1^2)\|x\|_2^2)\\
&+\mathbb E_{x\sim\mathcal X}(((C_2+\Delta_2)^2-C_2^2)\|x\|_2^2) \\
=& 2(C_1\Delta_1+C_2\Delta_2)\mathbb E_{x\sim\mathcal X}(\|x\|_2^2)\\
&+(\Delta_1^2+\Delta_2^2)\mathbb E_{x\sim\mathcal X}(\|x\|_2^2) \\
\leqslant& \left(A_1\left(DK_2\right)^2+A_2 \cdot DK_2\right)\mathbb E_{x\sim\mathcal X}(\|x\|_2^2),
\end{align*}
where $A_1,A_2>0$ are two constants. That completes the proof.
\end{proof}

\begin{corollary}
Let $f_P$ be a fully convolutional network with kernels $\{k_{1},\cdots,k_{l}\}$. The perceptual loss, including the VGG loss~\cite{simonyan2014very} and LPIPS loss~\cite{zhang2018unreasonable}, can be expressed as:
\begin{equation*}
L_{pcpt}(f^j, f_P) = \mathbb E_{x\sim\mathcal X}(\|f_P(f^j(x)) - f_P(y)\|_1),
\end{equation*}
where $y$ is the ground truth. Then there exist constants $C_{pcpt}>0$ such that the following inequality holds:
\begin{align*}
&|L_{pcpt}(f^1, f_P) - L_{pcpt}(f^2, f_P)|\\
\leqslant &C_{pcpt} (\sum_{i=1}^n\|k_i^1-k_i^2\|_1)E_{x\sim\mathcal X}(\|x\|_1).
\end{align*}
\yzp{Specially, if $f_P$ degenerates into an identity transformation, the perceptual loss degenerates into $L_1$ loss in Eq.~\ref{eq:rec-loss}.}
\end{corollary}
\begin{proof}
It is a direct consequence of Theorem~\ref{theorem:bound}.
\end{proof}

\begin{corollary}
Let $f_{\mathcal X\rightarrow\mathcal Y}^j$ and $f_{\mathcal Y\rightarrow\mathcal X}^j$, $j=1,2$, be the generators from the domain $\mathcal X$ to the domain $\mathcal Y$ and from $\mathcal Y$ to $\mathcal X$, respectively, which are fully convolutional networks. The cycle loss term~\cite{zhu2017unpaired} can be expressed by
\begin{equation}
L_{cyc}(f^j)=\mathbb E_{x\sim\mathcal X}(\|f^j_{\mathcal Y\rightarrow\mathcal X}(f^j_{\mathcal X\rightarrow\mathcal Y}(x)))-x\|_1).
\end{equation}
Then by considering $f^j_{\mathcal Y\rightarrow\mathcal X}\circ f^j_{\mathcal X\rightarrow\mathcal Y}$
as a $2n$-layer fully convolutional network, we have:
\begin{align*}
&|L_{cyc}(f^1)-L_{cyc}(f^2)|\leqslant\\
&C_{cyc}\left(\sum_{i=1}^n(\|Dk_{i,\mathcal X\rightarrow\mathcal Y}\|_1+\|Dk_{i,\mathcal Y\rightarrow\mathcal X}\|_1)\right)\mathbb E_{x\sim\mathcal X}(\|x\|_1),
\end{align*}
where $C_{cyc}$ is a constant, $Dk_{i,\mathcal X\rightarrow\mathcal Y} = k_{i,\mathcal X\rightarrow\mathcal Y}^1-k_{i,\mathcal X\rightarrow\mathcal Y}^2$ and $Dk_{i,\mathcal Y\rightarrow\mathcal X} = k_{i,\mathcal Y\rightarrow\mathcal X}^1-k_{i,\mathcal Y\rightarrow\mathcal X}^2$.
\end{corollary}

\begin{proof}
We have
\begin{align*}
&|L_{cyc}(f^1)-L_{cyc}(f^2)| \\
\leqslant& \mathbb E_{x\sim\mathcal X}(\big|\|f^1_{\mathcal Y\rightarrow\mathcal X}\circ f^1_{\mathcal X\rightarrow\mathcal Y}(x)-x\|_1 \\
&- \|f^2_{\mathcal Y\rightarrow\mathcal X}\circ f^2_{\mathcal X\rightarrow\mathcal Y}(x)-x\|_1\big|) \\
\leqslant& \mathbb E_{x\sim\mathcal X}(\|f^1_{\mathcal Y\rightarrow\mathcal X}\circ f^1_{\mathcal X\rightarrow\mathcal Y}(x)-f^2_{\mathcal Y\rightarrow\mathcal X}\circ f^2_{\mathcal X\rightarrow\mathcal Y}(x)\|_1) \\
\leqslant& C_{cyc}\left(\sum_{i=1}^n(\|Dk_{i,\mathcal X\rightarrow\mathcal Y}\|_1+\|Dk_{i,\mathcal Y\rightarrow\mathcal X}\|_1)\right)\mathbb E_{x\sim\mathcal X}(\|x\|_1).
\end{align*}
That completes the proof.
\end{proof}

\begin{figure*}[htb]
  \centering
  \includegraphics[width=\textwidth]{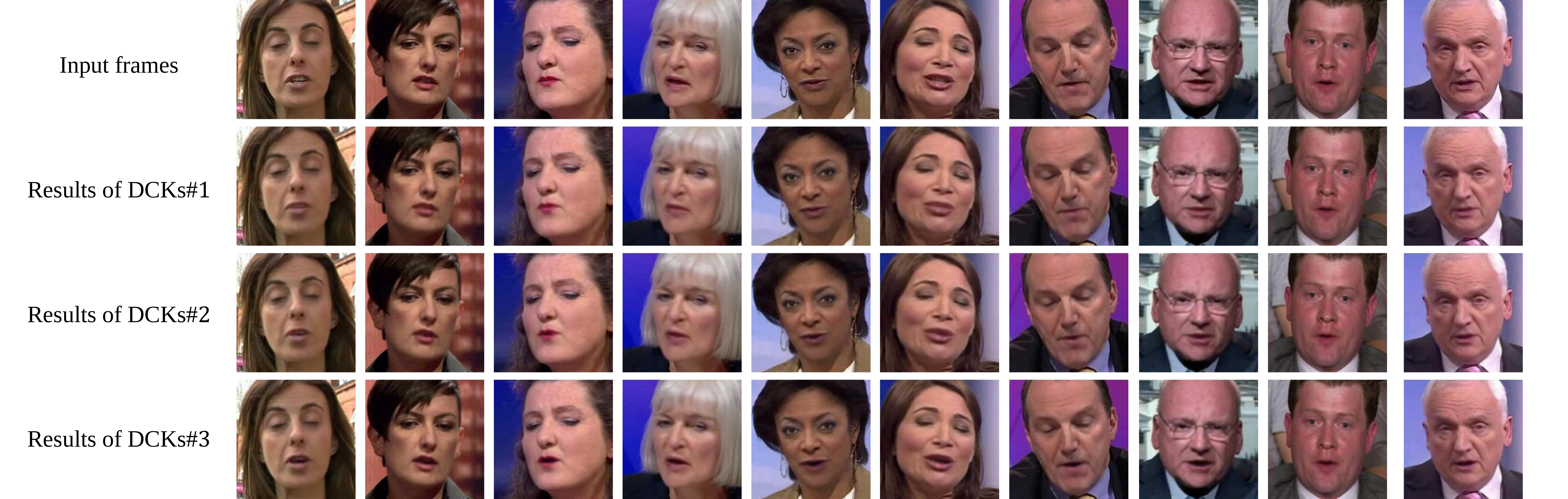}
  \caption{
Given different audio inputs \yzp{(successive 3 frames of the word `for')}, our system $f'$ can infer different DCKs (i.e., DCKs\#1,\#2 and \#3), such that $f'$ with DCKs\#1,\#2 and \#3 can transfer a face (with arbitrary expression in any input frame) into expressions with different mouth shapes.  }
 \label{fig:validation_explanation}
\end{figure*}

\section{Experiments}
\label{sec:experiment}

\subsection{Implementation Details}

We implemented our method with PyTorch~\cite{paszke2017automatic} and OpenCV. We trained the model on a server with an Intel Xeon Gold 6126 (2.60 GHz) and a NVIDIA TITAN RTX GPU. We also tested it on the same server. \yzp{We use a mixed video dataset (including real videos and synthetic videos) described in Sec.~\ref{subsec:trainning} to train our model.}

Our system starts with video pre-processing (Figure~\ref{fig:pipeline}). For a video with background, we crop the facial area (detected by Dlib~\cite{king2009dlib}) from the video and resize it to $256 \times 256$ as the input of our system. At the end of the pipeline, we cover the facial area with generated results directly without image fusion, which is a major advantage of our method that helps achieve real-time performance.

\yzp{We use a U-net with DCKs, called Adapted U-net, as the generator, which has $5$ down-sampling layers, $4$ middle layers and $5$ up-sampling layers, where all middle layers are with DCKs. We use the pre-trained audio network~\cite{Wav2Lip}, which consists of 2D convolutional layers and residual blocks, to extract audio features from Mel Spectrogram of input audio, where the parameters of Mel Spectrogram are the same as \cite{Wav2Lip}. For each layer with a dynamic convolution kernel, we train a fully connected network to infer the DCK from the audio features.}
We reshape the output of this module to the shape $l \times c_1 \times (c_2 \times ks \times ks + 1)$, where $l$ is the length of video sequence, $ks$ is the kernel size, $c_1$ and $c_2$ are the numbers of channels of output and input of the corresponding convolution layers of the adapted U-net. In all our experiments, $ks=1$, $c_1 = 256$ and $c_2 = 256$. We implement the DCKs by convolution operators with the group parameter in PyTorch.

\begin{figure}[ht]
\centering
\includegraphics[width=\columnwidth]{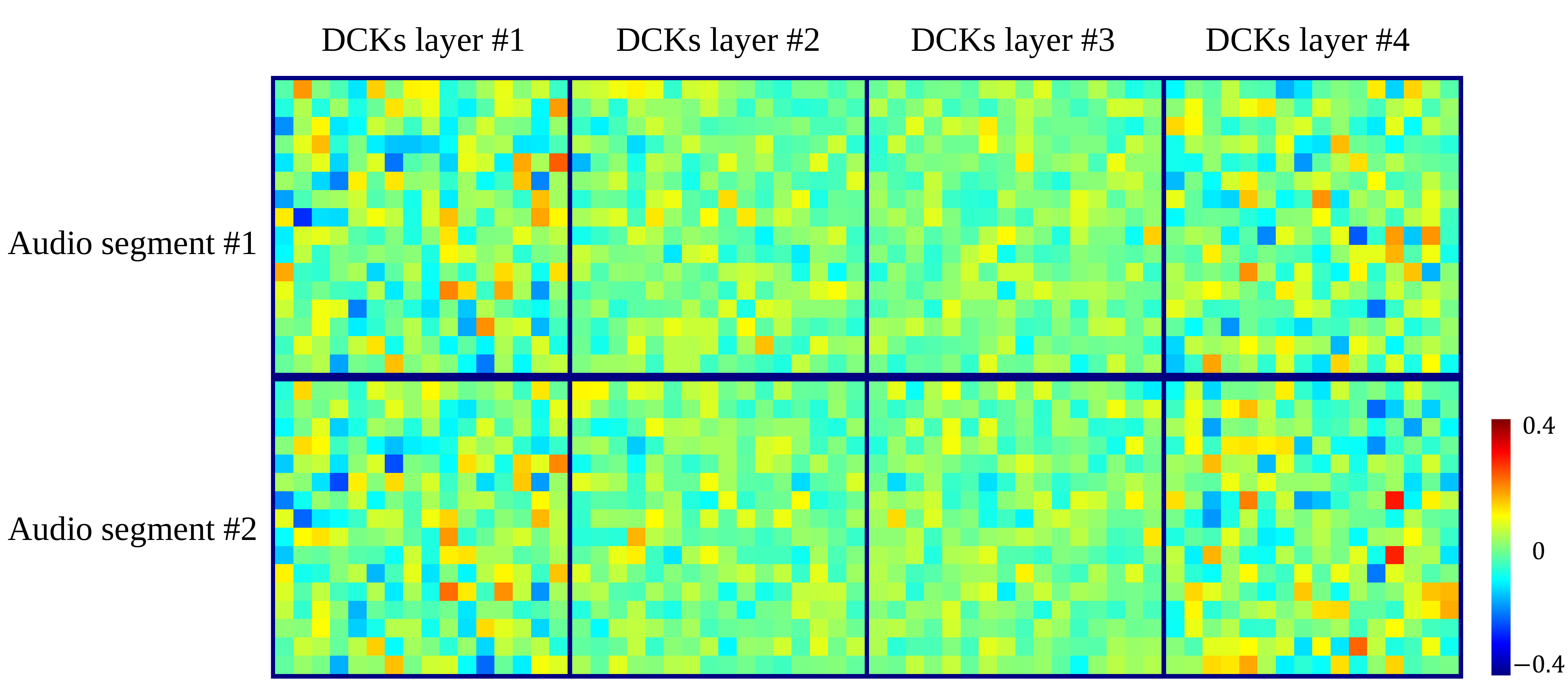}
\caption{Four DCKs layer inferred from 
two audio segments are showed. Different audio inputs lead different parameters of convolution kernel. \yzp{For better observation, we only visualize part of parameters of DCKs (i.e., the first $16 \times 16$ parameters).}}
\label{fig:DCKs_layer_show}
\end{figure}

\subsection{Validation of DCK understanding}
\label{subsec:DCK-validation}

Our system is a fully convolutional network $f'$ with DCKs. The theoretical interpretation in Section \ref{sec:DCK-interpretation} indicates that our system $f'$ can well approximate a set of networks $\{f^i\}$ with fixed parameters, such that according to different audio input, the system $f'$ can adaptively choose the desired $f\in\{f^i\}$. 
Some results are shown in Figure~\ref{fig:validation_explanation}. Given different audio inputs, our DCKs can successfully infer different parameters that drive the system $f'$ to approximate different $f\in\{f^i\}$, e.g., transforming a face (with arbitrary expression in any input frame) into other expressions with different mouth shapes. 
Another advantage of the system $f'$ with DCKs is that for any finite set of $\{f^i\}$, each input frame can only be transferred into other expressions in a finite expression space, while our DCKs can infer parameters in a continuous space, so that our system can provide better mechanism to secure the inter-frame continuity.

\begin{figure}[ht]
\centering
\includegraphics[width=\columnwidth]{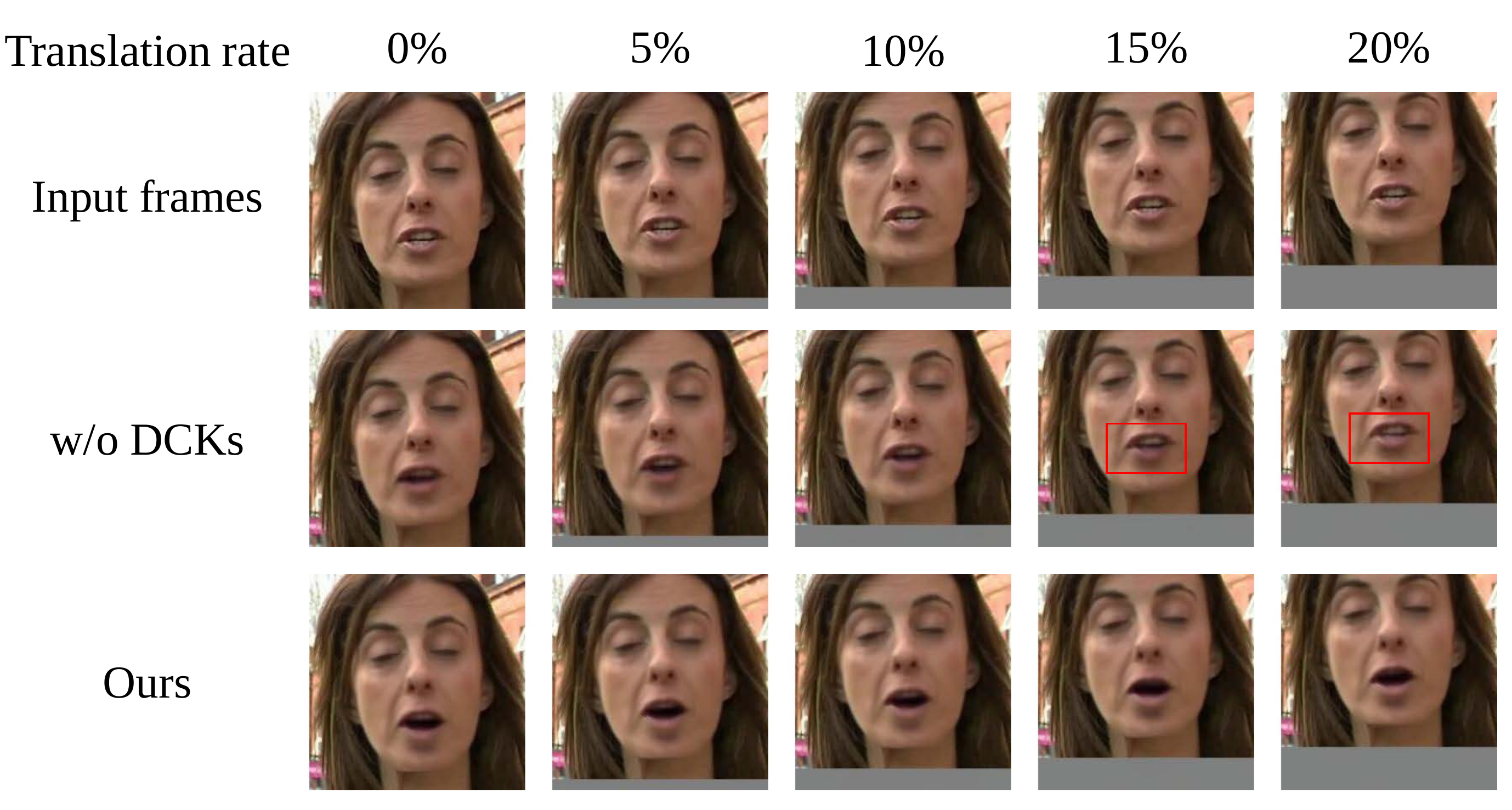}
\caption{Ablation study for comparing $f_{w/o\_DCKs}$ and our system $f'$. \yzp{The phoneme of input audio is /ba:/ of the word `Obama'.} See text for details.}
\label{fig:ablation_invariance}
\end{figure}

\begin{figure}[ht]
\centering
\includegraphics[width=\columnwidth]{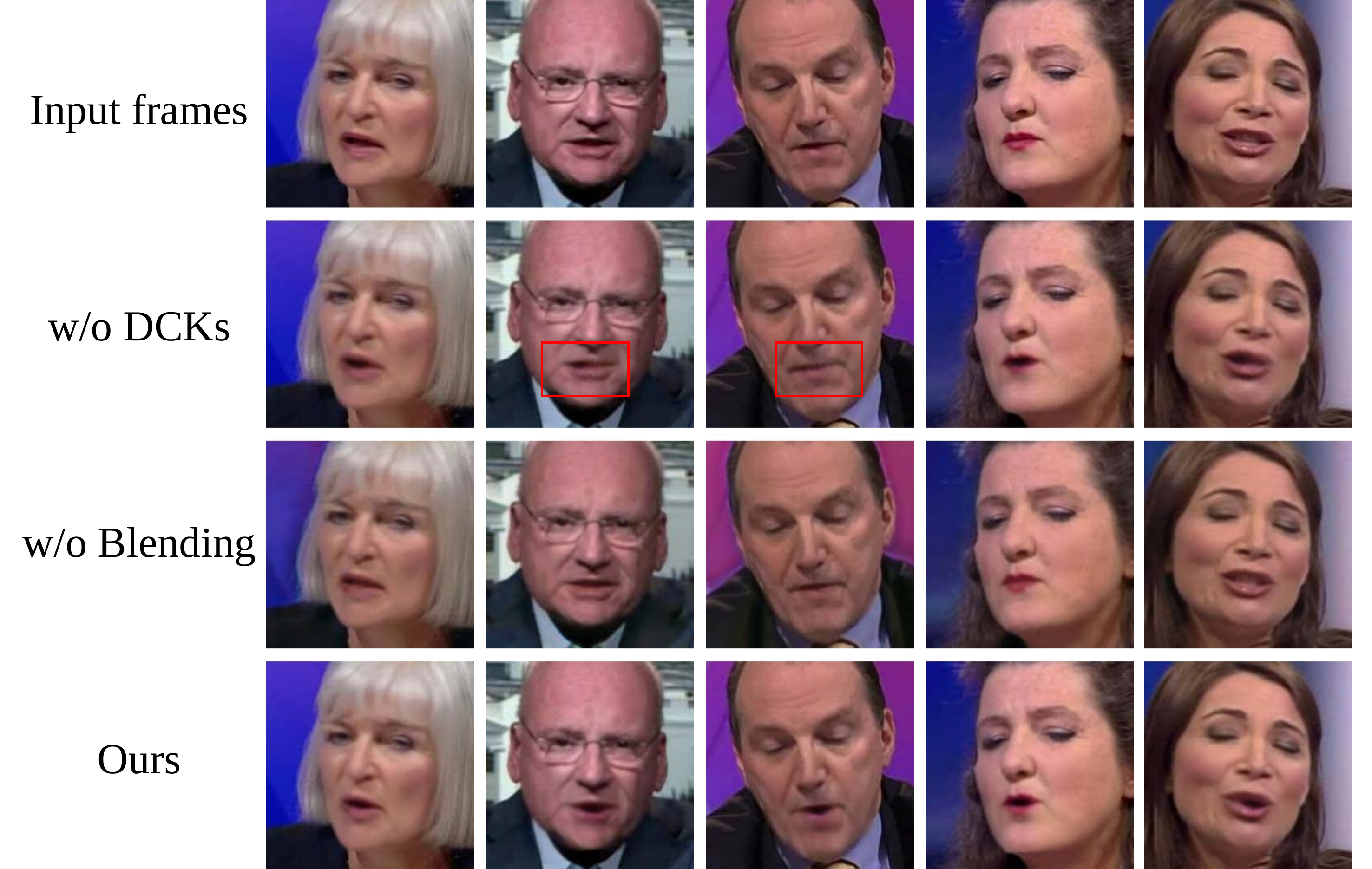}
\caption{Ablation study. The first row shows several frames of real videos as input frames. We use the same audio features as input audio for all the input frames. \yzp{The syllable of input audio is the word `were'. The second to the fourth rows show the generation results of our method with one of the modules disabled along with our full model. The generation results without DCKs sometimes have wrong lip motion. The generation results of directly generating the synthetic frames without blending are almost the same as input frames (mouth shape in particular), which shows our method without blending (i.e. the attention mechanism) is difficult to generate good talking-face videos.} Only the whole method can generate good results in all cases.}
\label{fig:ablation_study}
\end{figure}

{
\subsection{Visualization of DCKs}

A fully convolutional network with DCKs is a black box. Visualization of DCKs helps us understand it how to work. In all our experiment, kernel sizes of DCKs are $1$ so we can visualize each DCKs layer inferred from an audio segment as an image whose weight and height are the numbers of input and output channels of the DCKs layer. The DCKs layer could be regarded as the correlation coefficient between different channels, i.e. their covariance matrix. Four DCKs layer inferred from 
two audio segments are showed in Figure~\ref{fig:DCKs_layer_show}. Results show different audio inputs lead different parameters of convolution kernel.
}

\subsection{Ablation Study}
\label{subsec:ablation}

The ablation study focuses on the novel DCKs and blending. First, we design a network $f_{w/o\_DCKs}$ without DCK, which \yzp{reshapes the audio features to $16 \times H \times W$ (i.e., 16 channels and the same resolution $H \times W$ as the image), and use the spatial attention fusion module of the method~\cite{zhang2020davd} to fuse audio features and image features.} 
As shown in Figure~\ref{fig:ablation_invariance}, we compare $f_{w/o\_DCKs}$ with our fully convolutional network $f'$ with DCKs in the following test scenario: the face region in the input frame is moved upwards (defined by a translation rate which is the number of upwards translation pixels over the frame height). The results show that $f'$ is invariant for translation and outputs similar results for different translation rates, while $f_{w/o\_DCKs}$ \yzp{outputs different results for different translation rates, which leads to bad lip synchronization for slightly larger translation rate (Figure~\ref{fig:ablation_invariance} middle row, $15\%$ and $20\%$)}. The reason is possibly that reshaping the audio features to \yzp{image space} requires fixed semanteme at each pixel.

\begin{table}[ht]
\centering
\caption{Quantitative Comparison between ours and ablation study methods.} 
\begin{tabular}{c|c|c|c}
\hline
  Metric & w/o Blending & w/o DCKs & Ours \\
  \hline
  PSNR $\uparrow$ & 29.29 & 31.08 & {\bf 31.98}\\
  \hline
  SSIM $\uparrow$ & 0.74 & 0.78 & {\bf 0.81}\\
  \hline
  LMD $\downarrow$ & 1.65 & 1.61 & {\bf 1.44} \\
  \hline
\end{tabular}
\label{tab:ablation_quantitative_comparison}
\end{table}

\yzp{To demonstrate the effectiveness of our blending scheme}, we train a network to directly generate frames without blending. The qualitative results are shown in Figure~\ref{fig:ablation_study}. We use the same audio features as input audio for all the input frames, \yzp{whose syllable is the word `were'}. We compare the generation results of our method with and without the two modules, i.e. DCKs and blending. \yzp{The generation results without DCKs sometimes have wrong lip motion. The generation results of directly generating the synthetic frames without blending are almost the same as input frames (mouth shape in particular), which shows our method without blending (i.e. the attention mechanism) is difficult to generate good talking-face videos.} Only the whole method can generate good results in all cases.

\yzp{We use the test set of LRW dataet~\cite{ChungZ16} for quantitative metric evaluation. For each video in the test set, we input its first frame and audio signal to the network, and generate a talking-face video for each comparison method. We compare the results with the ground-truth videos, after aligning them according to the way used in ATVGnet~\cite{ChenMDX19}. We use Peak Signal to Noise Ratio (PSNR) and Structural Similarity Index Metrics (SSIM) to evaluate the quality of images, and use Landmark Distance (LMD)~\cite{ChenLMDX18} to evaluate the accuracy of lip movement. The results of quantitative comparison are summarized in Table~\ref{tab:ablation_quantitative_comparison}. We can see that DCKs and blending are helpful for generating talking-face videos.}

{
\subsection{Comparison with State of the Arts}
\label{subsec:stoa}

In this section, we compare our model with state-of-the-art methods, including ATVGnet~\cite{ChenMDX19}, You Said That~\cite{ChungJZ17}, Wav2Lip~\cite{Wav2Lip}, X2Face~\cite{WilesKZ18}, DAVS~\cite{ZhouLLLW19}, SDA~\cite{VougioukasPP18}, Yi's Method~\cite{yi2020audio}. We first introduce and discuss these methods.


\textbf{ATVGnet.} ATVGnet~\cite{ChenMDX19} generates talking-face video in real time from input photo and audio by hierarchical networks. It crops the facial area from the input photo and aligns it by affine transformation based on facial landmarks extracted from Dlib~\cite{king2009dlib}. Because the alignment operation changes and fixes the view angle, results of ATVGnet are talking-face videos without background and head motion. The resolution of its generation results is $128 \times 128$, which is lower than ours. Its results has neither head motion nor eye blink.

\textbf{You Said That.} You Said That~\cite{ChungJZ17} generates talking-face video from input photo and audio using two CNNs to extract features from audio (spectrum) and photo separately and then concatenating them in channels. It aligns the input photo by applying spatial registration so results of You Said That are talking-face videos without background, head motion and eye blink. The resolution of its generation results is $112 \times 112$, which is lower than ours. 

\begin{figure*}[tb]
  \centering
  \includegraphics[width=0.96\textwidth]{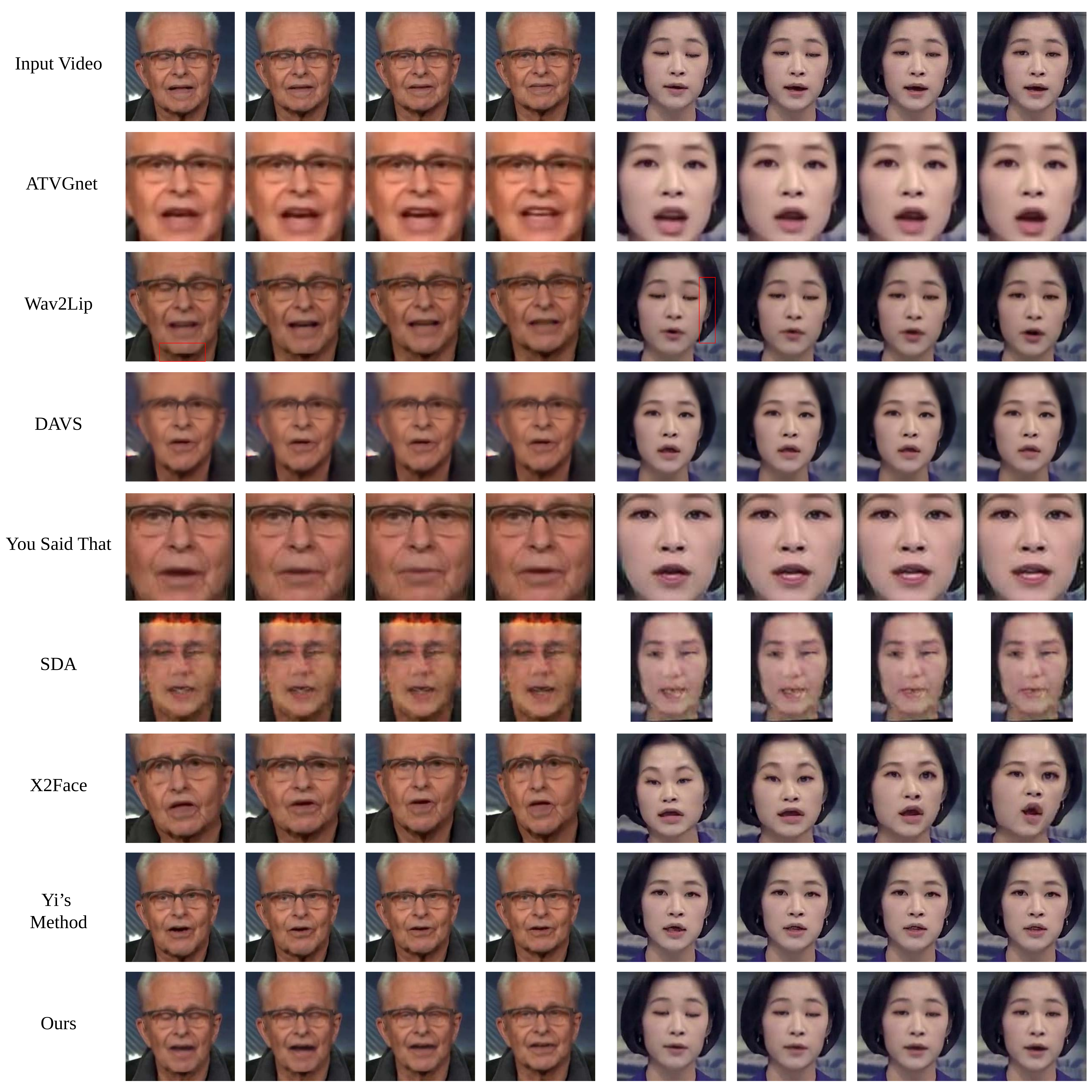}
\vspace{-0.05in}
\caption{
The qualitative comparison between the STOA methods (Wav2Lip~\cite{Wav2Lip}, X2Face~\cite{WilesKZ18}, You said that~\cite{ChungJZ17}, SDA~\cite{VougioukasPP18}, DAVS~\cite{ZhouLLLW19}, ATVGnet~\cite{ChenMDX19} and Yi's Method~\cite{yi2020audio}) and ours. Results of SDA are apparently worse than others. X2Face distorts the shape of face and changes the identity. ATVGnet, DAVS and You Said That results have no head motion and no eye blink. ATVGnet results are of low resolution and degraded visual quality. Wav2Lip results are of low definition in facial area, which sometimes causes an obvious boundary (red boxes) between facial part and other parts when directly covering the generated region onto the background. Yi's Method can not generate talking-head videos in real time. While our method can generate high-quality talking-head videos with head motion in real time. \yzp{The syllables of input audios are the phrase `I am' in the left and the word `really' in the right.}}
\label{fig:methods_compare}
\vspace{-0.05in}
\end{figure*}

\textbf{Wav2Lip.} Wav2Lip~\cite{Wav2Lip} generates talking-face video in real time from input photo (or input video) and audio using CNNs with encoder-decoder structure. It uses a lip-sync expert model to train the generation model in order to achieve high lip synchronization. However, the resolution of its generation results is $96 \times 96$, which is lower than ours.

\begin{table*}[ht]
\centering
\caption{Quantitative Comparison between ours and the STOA methods (Wav2Lip~\cite{Wav2Lip}, X2Face~\cite{WilesKZ18}, You said that~\cite{ChungJZ17}, SDA~\cite{VougioukasPP18}, DAVS~\cite{ZhouLLLW19}, ATVGnet~\cite{ChenMDX19} and Yi's Method~\cite{yi2020audio}). The number in bracket is the ranking of method.} 
\begin{tabular}{c|c|c|c|c|c|c|c|c}
\hline
  Metric & Wav2Lip~\cite{Wav2Lip} & X2Face~\cite{WilesKZ18} & YST~\cite{ChungJZ17} & SDA~\cite{VougioukasPP18} & DAVS~\cite{ZhouLLLW19} & ATVGnet~\cite{ChenMDX19} & Yi's Method~\cite{yi2020audio} & Ours \\
  \hline
  PSNR $\uparrow$ & 30.72 & 29.82 & 29.91 & 29.44 & 29.81 & 30.91 & 30.85 & \yzp{{\bf 31.98}} (1)\\
  \hline
  SSIM $\uparrow$ & 0.76 & 0.75 & 0.77 & 0.68 & 0.73 & {\bf 0.81} & 0.75 & {\bf 0.81} (1)\\
  \hline
  LMD $\downarrow$ & 1.61 & 1.60 & 1.63 & 2.32 & 1.73 & {\bf 1.37} & 1.58 & \yzp{1.44} (2)\\
  \hline
\end{tabular}
\label{tab:quantitative_comparison}
\end{table*}

\textbf{X2Face.} X2Face~\cite{WilesKZ18} controls a source frame using another frame with different identities to produce a generated frame with the identity of the source frame but the pose and expression of the other frame. It uses an auto-encoder to edit a frame in hidden space so the generation process can be driven by audio. However, the model can not totally disentangle different attributes of multi-modal inputs in hidden space which leads to unstable and discontinuous generation results. It can not keep the head pose so results of X2Face are talking-face videos without background.

\textbf{DAVS.} DAVS~\cite{ZhouLLLW19} generates talking-face video from input photo and audio using an auto-encoder to disentangle subject-related information and speech-related information via an associative-and-adversarial training process. However, the model can not totally disentangle speech-related information in hidden space, which leads to low audio-visual consistency. It cannot keep the boundary of the cropped face so results of DAVS are talking-face videos without background. It also has neither head motion nor eye blink. 

\textbf{SDA.} SDA~\cite{VougioukasPP18} generates talking-face video from input video and audio using a temporal GAN with 2 discriminators, i.e. frame discriminator and sequence discriminator, which are designed for different aspects of a video. However, the quality of generation results decreases over time. It can not keep the boundary of the cropped face so results of SDA are talking-face videos without background. The resolution of its generation results is $96 \times 128$, which is lower than ours.

\textbf{Rendering-based Methods.} In recent years, several rendering-based methods (e.g.~\cite{yi2020audio, song2020everybody, thies2020neural, wen2020photorealistic}) for generating talking-face video from a pair of unmatched video and audio have been proposed. Their pipelines are similar, which use the 3D parametric model as a prior, whose parameters are composed of identity components and expression components. They recover identity parameters from the video input by face reconstruction methods and predict expression parameters from the audio input by neural networks. Then they render the reconstructed face model and obtain a rendered frame. Most methods are trained for only one person, and need a large number of training data of one specified person. Only Yi's method~\cite{yi2020audio} trained a general model and works for arbitrary identity. The model is trained on a dataset with many identities and fine-tuned on a short video of the person. It also works without fine-tuning. Therefore, we only compare our method with Yi's method. In quantitative comparison, Yi's Method is trained on LRW dataset without fine-tuning. 

\textbf{Qualitative Comparison.} The qualitative comparison between the STOA methods (Wav2Lip~\cite{Wav2Lip}, X2Face~\cite{WilesKZ18}, You said that~\cite{ChungJZ17}, SDA~\cite{VougioukasPP18}, DAVS~\cite{ZhouLLLW19}, ATVGnet~\cite{ChenMDX19} and Yi's Method~\cite{yi2020audio}) and ours is shown in Figure~\ref{fig:methods_compare}. Results of SDA are apparently worse than others. X2Face distorts the shape of face and changes the identity. The results of ATVGnet, DAVS and You Said That have no head motion and no eye blink. ATVGnet results are of low resolution and degraded visual quality. Wav2Lip results are of low definition in facial area, which sometimes causes an obvious boundary between facial part and other parts when directly covering the generated region onto the background. Yi's Method can not generate talking-head videos in real time. While our method can generate high-quality talking-head videos with head motion in real time.

\textbf{Quantitative Comparison.} We use the test set of LRW dataet~\cite{ChungZ16} for quantitative metric evaluation. For each video in the test set, we take its first frame and audio signal as inputs, and generate a talking-face video for each comparison method. We compare the results with the ground-truth videos, after aligning them according to the way used in ATVGnet. We use Peak Signal to Noise Ratio (PSNR) and Structural Similarity Index Metrics (SSIM) to evaluate the quality of images, and use Landmark Distance (LMD) to evaluate the accuracy of lip movement. The results of quantitative comparison between ours and the SOTA methods (Wav2Lip~\cite{Wav2Lip}, X2Face~\cite{WilesKZ18}, You said that~\cite{ChungJZ17}, SDA~\cite{VougioukasPP18}, DAVS~\cite{ZhouLLLW19}, ATVGnet~\cite{ChenMDX19} and Yi's Method~\cite{yi2020audio}) are summarized in Table~\ref{tab:quantitative_comparison}, showing that our method has better performance than most of the SOTA methods on the three metrics above. 

\textbf{Perceptual Study.} There is no universal metric to evaluate the visual quality of generated video. The metrics used above are also limited in predicting visual quality. Therefore, it is a good way to use a perceptual study for measuring the visual quality. \yzp{We collected $10$ videos in the wild with different head poses and $10$ audios as inputs, and combined them to generate $100$ videos. We used ATVGnet, Wav2Lip, X2Face, YST and our method to generate talking-face videos without background.} A group of \yzp{five} videos generated from the same input was presented in a random order to the participants and they were asked to select the video with the best visual quality (VQ), lip synchronization (LS), inter-frame continuity (IFC) and overall quality (Overall): (1) visual quality is to measure the definition and naturalness of videos, (2) lip synchronization is to measure the correspondence between lip movements and audios, (3) inter-frame continuity is to measure continuity between successive frames of videos, and (4) overall quality is to measure videos by combining all the three metrics. \yzp{$20$ participants attended the perceptual study and each of them compared $20$ random groups of video. The 
statistics of the user study are summarized in Table~\ref{tab:perceptual_study}. Our method has the best visual quality, lip synchronization, inter-frame continuity and overall quality. }

\begin{table}[ht]
\centering
\caption{Perceptual study on visual quality,  lip synchronization, and inter-frame continuity.}  
\begin{tabular}{c|c|c|c|c}
\hline
Methods                & VQ & LS & IFC & Overall \\ \hline
ATVGnet~\cite{ChenMDX19}          & 1.0\%  & 1.5\% & 1.0\% & 1.0\% \\ \hline
Wav2Lip~\cite{Wav2Lip}            & 18.0\% & 41.2\% & 29.5\% & 24.5\% \\ \hline
X2Face~\cite{WilesKZ18}           & 0.0\%  & 0.2\% & 0.0\% & 0.0\% \\ \hline
YST~\cite{ChungJZ17}              & 0.8\% & 1.2\%  & 2.0\% & 1.5\%  \\ \hline
Ours                   & {\bf 80.2\%} & {\bf 55.8\%} & {\bf 67.5\%} & {\bf 73.0\%} \\ \hline
\end{tabular}
\label{tab:perceptual_study}
\end{table}

In addition to the figures illustrated in this section, video examples are presented in the accompanying demo video. 

\begin{table*}[ht]
\centering
\caption{Running Time of ours and the STOA methods (Wav2Lip~\cite{Wav2Lip}, X2Face~\cite{WilesKZ18}, You said that~\cite{ChungJZ17}, SDA~\cite{VougioukasPP18}, DAVS~\cite{ZhouLLLW19}, ATVGnet~\cite{ChenMDX19} and Yi's Method~\cite{yi2020audio}). You said that, ATVGnet and X2Face change the head pose so they cannot generate results with background. SDA and DAVS cannot keep the boundary of the cropped face so they also cannot generate results with background.} 
\begin{tabular}{c|c|c|c|c|c|c|c|c}
\hline
  Metric & Wav2Lip~\cite{Wav2Lip} & X2Face~\cite{WilesKZ18} & YST~\cite{ChungJZ17} & SDA~\cite{VougioukasPP18} & DAVS~\cite{ZhouLLLW19} & ATVGnet~\cite{ChenMDX19} & Yi's Method~\cite{yi2020audio} & Ours \\
  \hline
  Time Cost per Frame (w/o BG) (ms)   & 9.9 & 95.0 & 88.9 & 32.5 & 107.2 & 6.7 & 567.4 & {\bf 4.7}\\
  \hline
  Maximum FPS (w/o BG) & 100 & 11 & 11 & 31 & 9 & 16 & 2 & {\bf 212}\\
  \hline
  Maximum FPS (with BG) & 42 & / & / & / & / & / & 1 & {\bf 60}\\
  \hline
\end{tabular}
\label{tab:running_time}
\end{table*}

\subsection{Running time}

Our method can generate talking-face videos in real time. For generating a $6 s$ video with $159$ frames, it takes $0.75 s$ in generating facial video with $256 \times 256$ resolution and it takes $2.62 s$ in total in generating a video with $1280 \times 720$ background. For generating a video with the same input, (1) ATVGnet~\cite{ChenMDX19} takes $1.07 s$ to generate video without background, (2) You Said That~\cite{ChungJZ17} takes $14.13 s$ to generate video without background, (3) X2Face~\cite{WilesKZ18} takes $15.10 s$ to generate video without background, (4) DAVS~\cite{ZhouLLLW19} takes $17.05 s$ to generate video without background, and (5) SDA~\cite{VougioukasPP18} takes $5.17 s$ to generate video without background. Yi's method~\cite{yi2020audio} takes $90.22 s$ to generate video without background and $127.24 s$ to generate video with background, which also takes about $1$ hour to fine-tune their network. Wav2Lip~\cite{Wav2Lip} takes $1.58 s$ to generate video without background and takes $3.83 s$ to direct cover it with background. However, sometimes the definition of facial area is lower than other parts which causes an obvious boundary between foreground and background. Therefore, only our method and Wav2Lip can generate talking-face video with background in real time (over $25$ fps). We show running time of all above methods in TABLE~\ref{tab:running_time} for a direct comparison. 
}

\section{Conclusion}

We propose a novel fully convolutional network with DCKs for the multi-modal task of audio-driven taking face video generation. Our simple yet effective system can generate high-quality talking-face video from unmatched video and audio in real time. Our solution is end-to-end, one-for-all and robust to different identities, head postures and audios. For preserving identities in both input and output talking-head videos, we propose a \yzp{novel supervised} training scheme. The results show our method can generate high-quality $60$ fps talking-head video with background in real time. Comparison and evaluation between our method and state-of-the-art methods show that our method achieves a good balance between various criteria such as running time, qualitative and quantitative qualities. Our novel DCK technique can potentially be applied to other multi-modal generation tasks, and meanwhile, our theoretical interpretation of DCK can be extended from fully convolutional network to forward networks involving ResNet modules, which we will investigate in future work.


%


\ifCLASSOPTIONcompsoc
  \section*{Acknowledgments}
\else
  \section*{Acknowledgment}
\fi

This work was supported by the Natural Science Foundation of China (61725204), Tsinghua University Initiative Scientific Research Program and Shanghai Municipal Science and Technology Major Project (2021SHZDZX0102).

\ifCLASSOPTIONcaptionsoff
  \newpage
\fi

\bibliographystyle{IEEEtran}
\bibliography{dcks}



\begin{IEEEbiography}[{\includegraphics[width=1in,height=1.25in,clip,keepaspectratio]{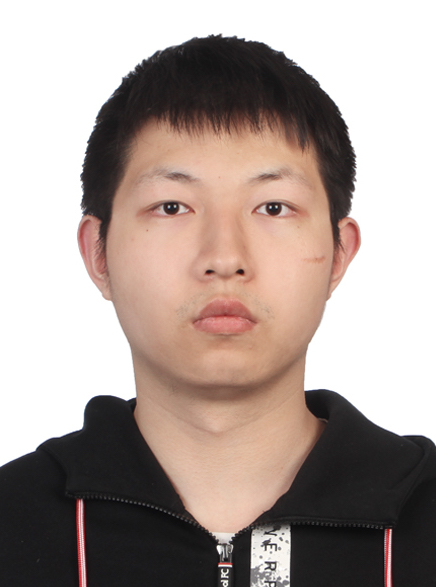}}]{Zipeng Ye} 
is a Ph.D student with Department of Computer Science and Technology, Tsinghua University. He received her B.Eng. degree from Tsinghua University, China, in 2017. His research interests include computational geometry, computer vision and computer graphics.
\end{IEEEbiography}

\begin{IEEEbiography}[{\includegraphics[width=1in,height=1.25in,clip,keepaspectratio]{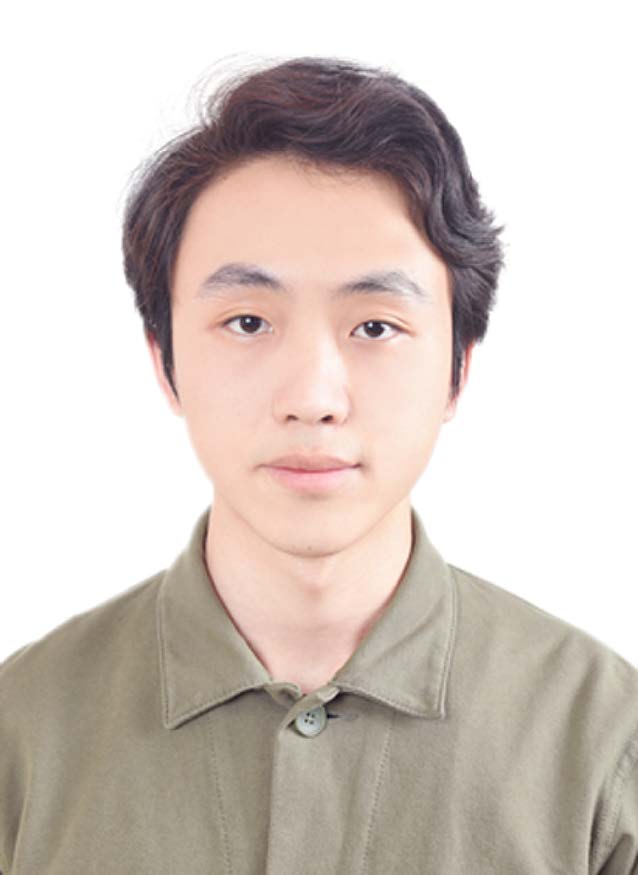}}]{Mengfei Xia}
is a Ph.D student with Department of Computer Science and Technology, Tsinghua University. He received his B.S. degree from Tsinghua University, China, in 2020.  He received her B.S. degree from Tsinghua University, China, in 2020. He won the silver medal twice in 30th and 31st National Mathematical Olympiad of China. His research interests include mathematical foundation in deep learning, image processing and computer vision.
\end{IEEEbiography}

\begin{IEEEbiography}[{\includegraphics[width=1in,height=1.25in,clip,keepaspectratio]{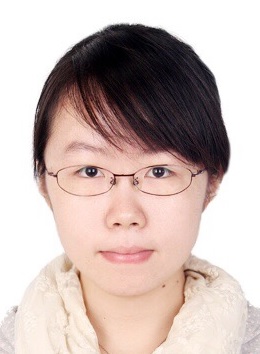}}]{Ran Yi}
is an assistant professor in the Department of Computer Science and Engineering, Shanghai Jiao Tong University. She received her B.Eng. degree and PhD degree from Tsinghua University, China, in 2016 and 2021. Her research interests include computational geometry, computer vision and computer graphics.
\end{IEEEbiography}

\begin{IEEEbiography}[{\includegraphics[width=1in,height=1.25in,clip,keepaspectratio]{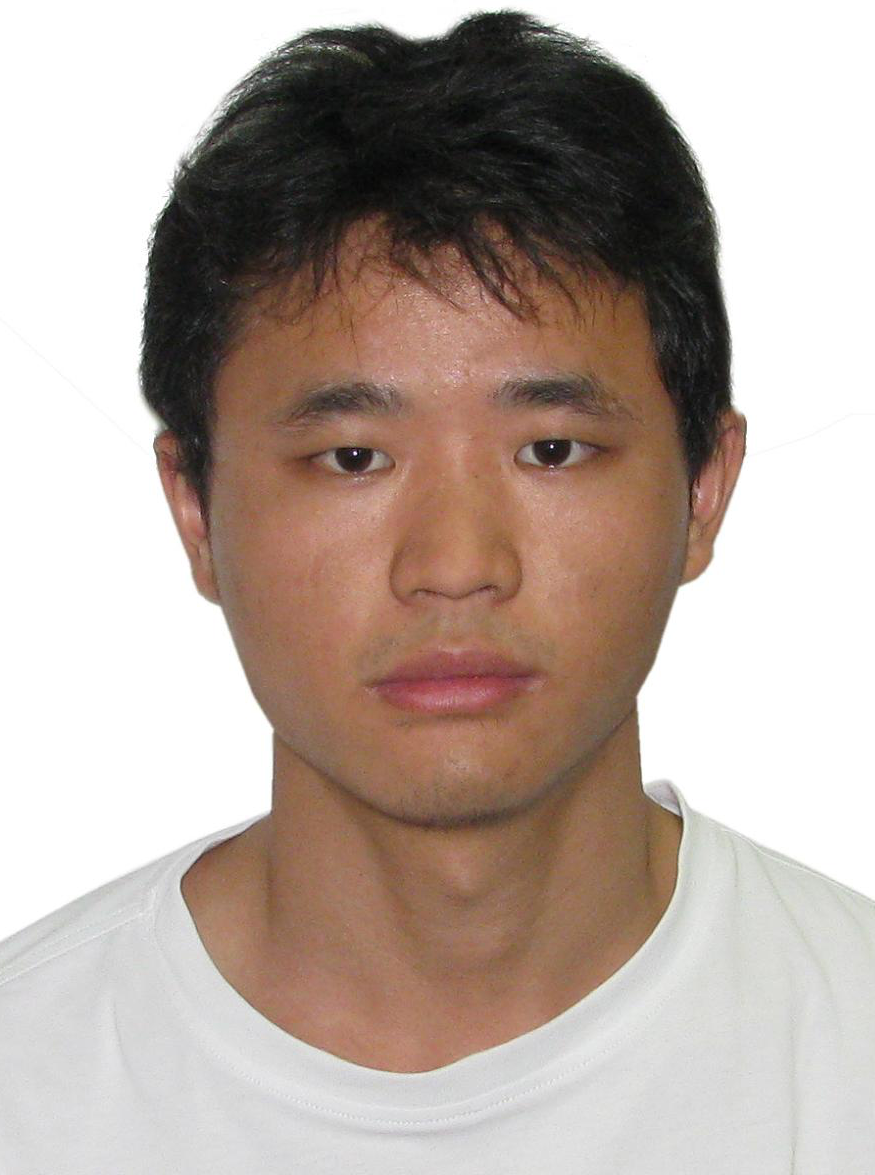}}]{Juyong Zhang}
is an associate professor in the School of Mathematical Sciences at University of Science and Technology of China. He received the BS degree from the University of Science and Technology of China in 2006, and the PhD degree from Nanyang Technological University, Singapore. His research interests include computer graphics, computer vision, and numerical optimization. He is an associate editor of The Visual Computer.
\end{IEEEbiography}

\begin{IEEEbiography}[{\includegraphics[width=1in,height=1.25in,clip,keepaspectratio]{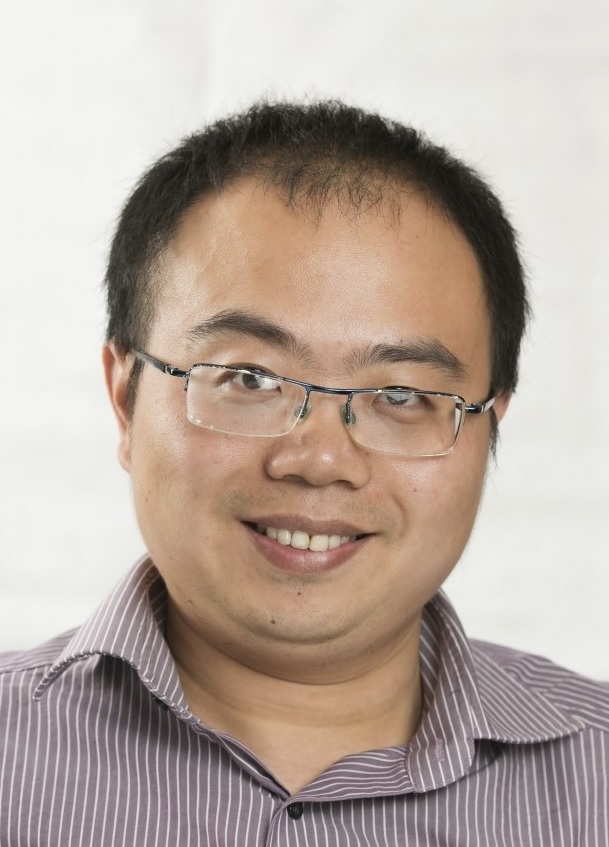}}]{Yu-Kun Lai}
is a Professor at School of Computer Science and Informatics, Cardiff University, UK. He received his B.S and PhD degrees in Computer Science from Tsinghua University, in 2003 and 2008 respectively. His research interests include computer graphics, computer vision, geometric modeling and image processing. For more information, visit \href{https://users.cs.cf.ac.uk/Yukun.Lai/}{https://users.cs.cf.ac.uk/Yukun.Lai/}
\end{IEEEbiography}

\begin{IEEEbiography}[{\includegraphics[width=1in,height=1.25in,clip,keepaspectratio]{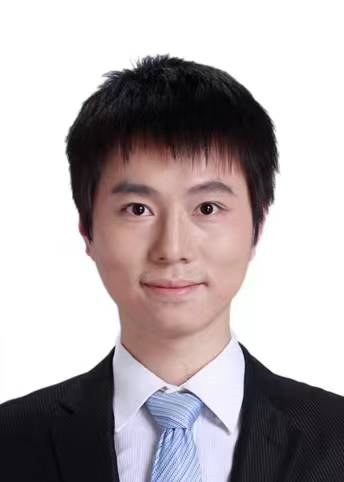}}]{Xuwei Huang}
received the bachelor degree in Information Engineer and the master degree in Electronic Engineer from South China University of Technology, GuangZhou, in 2012 and 2015, respectively. Currently, he is an AI algorithm engineer in Beijing Kuaishou Technology Co., Ltd. Since 2018, he has participated in many research and development works, including Facial Expression Recognition with the Animoji Technology, Speech synthesis Using few-shot learning, GAN-based live photo and Talking Face generation. He is responsible and focused on virtual human related technologies.
\end{IEEEbiography}

\begin{IEEEbiography}[{\includegraphics[width=1in,height=1.25in,clip,keepaspectratio]{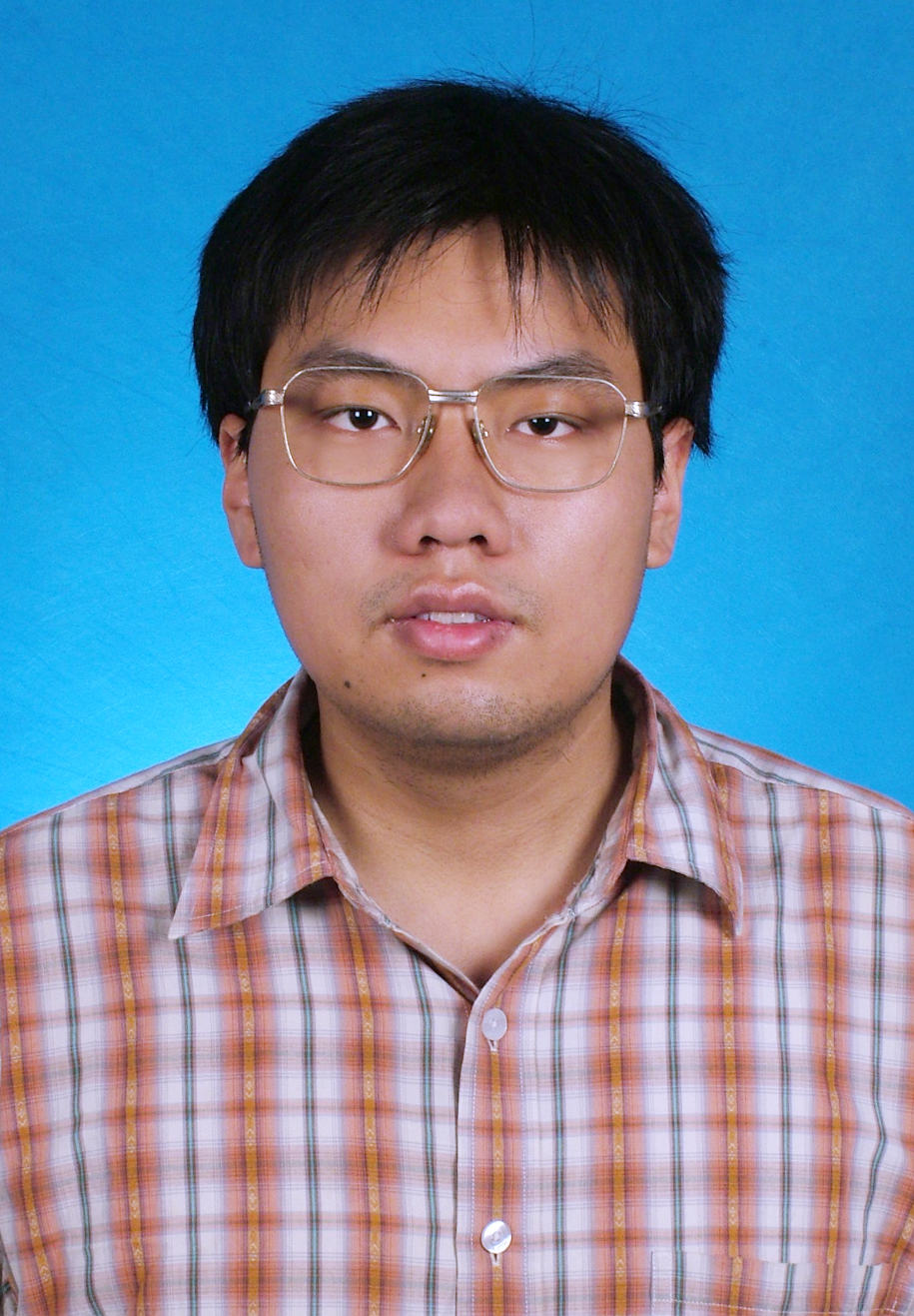}}]{Guoxin Zhang}
received his B.Eng and Ph.D. degree in the Department of Computer Science and Technology, Tsinghua University in 2007 and 2012, respectively. His research interests include computer graphics, geometric modeling, and computer vision, and he is currently working at Kwai Inc.
\end{IEEEbiography}

\begin{IEEEbiography}[{\includegraphics[width=1in,height=1.25in,clip,keepaspectratio]{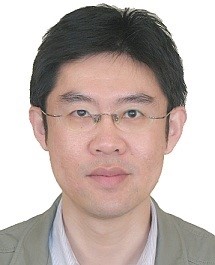}}]{Yong-Jin Liu}
is a Professor with the Department of Computer Science and Technology, Tsinghua University, China. He received the BEng degree from Tianjin University, China, in 1998, and the PhD degree from the Hong Kong University of Science and Technology, Hong Kong, China, in 2004. His research interests include computational geometry, computer graphics and computer vision. He is a senior member of the IEEE. For more information, visit \href{https://cg.cs.tsinghua.edu.cn/people/~Yongjin/Yongjin.htm}{http://cg.cs.tsinghua.edu.cn/people/$\sim$Yongjin/Yongjin.htm}
\end{IEEEbiography}

\end{document}